\documentclass[12pt, draftclsnofoot, onecolumn]{IEEEtran}
\usepackage{graphicx}
\usepackage{epsfig,subfigure}
\usepackage{amssymb,amsmath}
\usepackage{enumerate}
\usepackage{mathbbold}
\usepackage{cite,url,subfigure,epsfig,graphicx,wrapfig}
\usepackage{latexsym}
\usepackage{multirow}
\usepackage{array}
\usepackage{cite}
\usepackage{color, soul}
\usepackage[usenames, dvipsnames]{xcolor}
\usepackage{amsthm}
\usepackage{mathrsfs}
\usepackage{amsmath}

\newtheorem{assumption}{\textbf{Assumption}}
\newtheorem{lemma}{\textbf{Lemma}}
\newtheorem{theorem}{\textbf{Theorem}}
\newtheorem{corollary}{\textbf{Corollary}}

\usepackage{booktabs,arydshln}
\usepackage[noend]{algpseudocode}
\usepackage{algorithmicx,algorithm}

\begin{document}

\title{
Energy and Spectrum Efficient Federated Learning via High-Precision Over-the-Air Computation 
\author{
Liang Li, \emph{Member}, \emph{IEEE}, Chenpei Huang, \emph{Student Member}, \emph{IEEE}, Dian Shi, \emph{Member}, \emph{IEEE}, Hao Wang, \emph{Member}, \emph{IEEE}, Xiangwei Zhou, \emph{Senior Member}, \emph{IEEE}, Minglei Shu, \emph{Member}, \emph{IEEE}, and Miao Pan, \emph{Senior Member}, \emph{IEEE}
\thanks{L. Li is with the School of Computer Science (National Pilot Software Engineering School), Beijing University of Posts and Telecommunications, Beijing, 100876, China (e-mail: liliang1127@bupt.edu.cn).}
\thanks{C. Huang, D. Shi and M. Pan are with the Electrical and Computer
Engineering Department, University of Houston, TX, 77004, USA (e-mail: chuang25@uh.edu, dshi3@uh.edu, mpan2@uh.edu).}
\thanks{H. Wang is with the Division of Computer Science and Engineering, Louisiana State University, Baton Rouge, LA, 70803, USA (e-mail: haowang@lsu.edu).}
\thanks{X. Zhou is with the Division of Electrical and Computer Engineering, Louisiana State University, Baton Rouge, LA, 70803, USA (e-mail: xwzhou@lsu.edu).}
\thanks{M. Shu is with the Shandong Artificial Intelligence Institute, Qilu University of Technology (Shandong Academy of Sciences), Jinan, 250353, China (e-mail: shuml@sdas.org).}
}
}

\maketitle

\begin{abstract}
Federated learning (FL) enables mobile devices to collaboratively learn a shared prediction model while keeping data locally. However, there are two major research challenges to practically deploy FL over mobile devices: (i) frequent wireless updates of huge size gradients v.s. limited spectrum resources, and (ii) energy-hungry FL communication and local computing during training v.s. battery-constrained mobile devices. To address those challenges, in this paper, we propose a novel multi-bit over-the-air computation (M-AirComp) approach for spectrum-efficient aggregation of local model updates in FL and further present an energy-efficient FL design for mobile devices. Specifically, a high-precision digital modulation scheme is designed and incorporated in the M-AirComp, allowing mobile devices to upload model updates at the selected positions simultaneously in the multi-access channel. Moreover, we theoretically analyze the convergence property of our FL algorithm. Guided by FL convergence analysis, we formulate a joint transmission probability and local computing control optimization, aiming to minimize the overall energy consumption (i.e., iterative local computing + multi-round communications) of mobile devices in FL. Extensive simulation results show that our proposed scheme outperforms existing ones in terms of spectrum utilization, energy efficiency, and learning accuracy. 
\end{abstract}


\section{Introduction}
With the development of mobile communications and Internet-of-Things (IoT) technologies, mobile devices with built-in sensors and Internet connectivity have proliferated  huge volumes of data at the network edge. These data can be collected and analyzed to build increasingly complex machine learning models. To avoid raw-data sharing among the untrustworthy parties and leverage the ever-increasing computation capability of mobile devices, the emerging federated learning (FL) framework allows participating mobile devices to collaboratively train a machine learning model under the orchestration of a centralized server by just exchanging the local model updates with others via wireless communications. With such desirable properties, FL over mobile devices has inspired a wide utilization in a large variety of intelligent services, such as the keyword prediction~\cite{mcmahan2017communication}, voice classifier~\cite{Siri}, and e-health~\cite{brisimi2018federated}, etc.

Although only model updates instead of raw data are transferred between mobile devices and the FL server, such updates could contain hundreds of millions of parameters with complex neural networks. That makes the uplink transmissions from mobile devices to the FL server for model aggregation particularly challenging, resulting in a huge burden on both wireless networks and mobile devices. On the one hand, the spectrum resource that can be allocated to each device decreases proportionally as the number of devices increases, which hampers the scalability of FL to accommodate a large number of mobile devices with limited spectrum resources. On the other hand, transmitting a large volume of model updates periodically and executing heavy local on-device computations can quickly drain out the energy of battery-powered mobile devices. Such a mismatch restricts mobile devices or makes them reluctant to participate in FL.

Over-the-air computation (AirComp) provides a promising solution to address the aforementioned spectrum challenge by achieving scalable and efficient model update aggregation in FL. Unlike the conventional orthogonal multiple access techniques, where each user is restricted to its allocated spectrum band~\cite{pan2011infocom}, AirComp allows all the users to utilize the whole spectrum for transmissions simultaneously. By applying AirComp to FL, all the participating devices can transmit their model updates on the same channel. Due to the fact that mutli-access channel (MAC) inherently yields an additive superposed signal, the signals of all the participating devices are aligned to obtain desired arithmetic computation results directly over the air, thus significantly improving the spectrum efficiency. However, most works in the literature employ the analogy modulation to design their over-the-air FL schemes, which is not compatible with commercial off-the-shelf digital mobile devices and thus hinders their deployment in current/future communication systems, such as LTE, 5G, Wi-Fi 6, and 6G, etc. Besides, most existing efforts focus on single-iteration transmission design for AirComp-based FL~\cite{yang2020twcflota,fan2021joint}, and the impacts of AirComp on overall FL training performance, especially the FL convergence, are rarely discussed.

Therefore, in this paper, we design a novel multi-bit Aircomp (M-AirComp) FL scheme, named ESOAFL, which is compatible with the most common Quadrature Amplitude Modulation (QAM) to transmit the model updates, so that we do not need to modify the modulation protocols manufactured within commercial off-the-shelf mobile devices. Specifically, gradient quantization is incorporated into the ESOAFL scheme to facilitate the digital modulation, and only part of the gradients are selected to transmit to cope with the channel fading. In addition to handling the spectrum issue in FL over wireless networks, our scheme is battery-friendly to the participating mobile devices. Here, the energy consumption is considered from the long-term learning perspective where local computing (i.e., ``working'') and wireless communication (i.e., ``talking'') are two main focuses. Our M-AirComp FL scheme only requires updated gradients with good channel conditions to transmit, which further saves the communication energy compared with other AirComp schemes. Moreover, we theoretically analyze the convergence property of our ESOAFL approach, based on which we quantify the number of communication rounds needed for achieving the convergence, and the overall long-term energy consumption is further modeled. Finally, we develop a joint transmission probability and local computing control approach to balance ``working'' and ``talking,'' thus minimizing overall energy consumption. Our salient contributions are summarized as follows.
%Moreover, the optimal control, including gradients participating control and local computing control, are adopted to balance the working and talking, thus minimizing the overall energy consumption.

\begin{itemize}
    \item We propose an energy and spectrum efficient M-AirComp FL (ESOAFL) scheme where updated gradients are quantized into high-precision bitstreams, adapting to the digital modulation settings. To facilitate the M-AirComp, the transmission power of devices is carefully controlled to select the updates with good channel conditions for FL model aggregation.
    %Additionally, we adopt an energy efficient power control policy to facilitate the M-AirComp, where only updated gradients with good channel conditions are selected to participate in the FL training.
    
    \item We theoretically analyze the convergence of ESOAFL to characterize the impacts of the M-AirComp on FL. Guided by it, the gradients transmission probability and local computing iterations are jointly optimized from the long-term learning perspective, aiming to achieve energy-efficient federated training on mobile devices over spectrum-constrained wireless networks.
     
    %To help minimize the overall energy consumption of the proposed ESOAFL approach, the corresponding convergence analysis is derived, which quantitatively indicates the impacts of the M-AirComp on FL. Besides, the gradients transmission probability and local computing iterations are further optimized to achieve the overall energy efficiency.
    
    \item We conduct extensive simulations to verify the superiority of the ESOAFL compared to several baselines, under varying learning models, training datasets, and network settings. It shows that our ESOAFL scheme can improve spectral efficiency dozens of times and save at least half of the energy consumption.
    
    %We conduct extensive simulations and verify the efficacy of our ESOAFL scheme and the corresponding control approach under various learning models, datasets, and multiple wireless environmental settings. Compared with other schemes, our proposed method shows significant spectrum utilization and energy efficiency superiority. The ESOAFL approach has the potential to improve spectral efficiency dozens of times and save at least half of the energy consumption.
\end{itemize}

The rest of this paper is organized as follows. Section II provides some preliminaries of AirComp and FL. Section III describes our M-AirComp design and the corresponding ESOAFL approach. The convergence analysis of the proposed ESOAFL approach is derived in Section IV, and the formulation and solution of the energy efficient control scheme are also presented. Numerical simulations are provided in Section V, and VI reviews related works of the AirComp FL. Section VII finally concludes the paper and provides future work.

\section{Preliminaries of FL and AirComp FL}

\subsection{Preliminaries of FL}
We consider a federated learning system consisting of $K$ participating users carrying mobile devices in each FL round, where user $k \in \{1,2,\dots,K\}$ has its own data set, denoted by $\mathcal{D}_k$. The goal of FL is to collaborate the users to perform a unified optimization task, formally written as:
\begin{align}
    \min _{\mathbf{w} \in \mathbb{R}^{d}} f(\mathbf{w}) \triangleq \frac{1}{K} \sum_{k=1}^{K} f_{k}(\mathbf{w}),
\end{align}
where $f_{k}$ is the local loss function corresponding to user $k$, and $d$ is the dimension of the model parameters. 

Let $r\in\{1,2,\dots,R\}$ be the index of FL global communication round, and $H$ be the number of local computing iterations executed between every two consecutive global communication rounds. Moreover, we define $\mathbf{w}^r$ as the global model at the $r$-th communication round and define $\mathbf{w}^{r,h}_k$ as the local model of user $k$ at the $h$-th local iteration in the $r$-th communication round. Then the local updating process of user $k$ in the $r$-th communication round is given by:
\begin{align}
    \mathbf{w}^{r,h+1}_k = \mathbf{w}^{r,h}_k - \eta \nabla F_k(\mathbf{w}^{r,h}_k) \ \text{for} \ h=0,1,..., H-1, 
\end{align}
where $\nabla F_k(\mathbf{w}^{r,h}_k)$ is a stochastic gradient of function $f(\cdot)$ with a random batch-size data, and $\eta$ is the local learning rate. Here, $\nabla F_k(\mathbf{w}^{r,h}_k)$ is an unbiased estimation of $\nabla f_k(\mathbf{w}^{r,h}_k)$, i.e., $\mathbb{E}_{\xi \sim \mathcal{D}_{k}}\left[\nabla F_k(\mathbf{w}) \mid \boldsymbol{\xi}\right]= \nabla f_{k}(\mathbf{w})$, where $\xi$ represents the randomness like the batch-size index. After finishing the local training, every participating user uploads its local model updates to the server for global aggregation, i.e., $\eta\sum_{h=0}^{H-1} \nabla F_k(\mathbf{w}^{r,h}_k)$, and the server then broadcasts the most recent global model to initiate a new round of local training. The above process is repeated until the global model converges.

%During the transmission of local updates, we consider the gradient quantization and AirComp techniques to improve the spectrum efficiency, thus saving the communication resources. We implement $Q(\cdot )$ and $\text{Air}(\cdot)$ to represent the gradient quantization and AirComp operators, respectively. Specifically, $Q(\cdot )$ operator enables mobile devices to quantify the computed local gradients with specific bits, i.e., 8-bit or 16-bit, thus matching the digital modulation scheme we proposed. $\text{Air}(\cdot)$ operator aggregates the inputs from all users in our scenario, which will be introduced in detail in next section. 

\subsection{Preliminaries of AirComp FL}

During the FL training process, all the users have to transmit their local updates to the server for global aggregation, which may result in severe transmission congestion and consume a lot of communication resources, especially in cases with massive participating users. As one of the advanced wireless techniques, over-the-air computation (AirComp) enables all users to simultaneously transmit the local gradients over the same wireless medium without spectrum allocation and naturally aggregates the local updates during the signal propagation, which exhibits significant potentials to improve the spectrum utilization. 

Let $\mathcal{X}:=\{x_1,x_2,...,x_K\}$ and $\tilde{y}$ denote the input set and the output objective of the system, respectively. Here, $x_k$ is the gradients to be transmitted by user $k$, and $\tilde{y}$ is the global aggregation result received at the server with AirComp. Generally, an AirComp-based wireless communication system adopts precoding and amplification at transmitters, while receivers often have equalization blocks for signal detection. Therefore, AirComp computes the aggregated objective as %at each time slot
\begin{align}
\label{eq:AirCompAvg}
    \tilde{y} := \text{Air}\left(\mathcal{X}\right) = \frac{a}{K}\Bigg[\sum_{k=0}^{K} h_{k}p_{k}x_{k} + n\Bigg],
\end{align}
where $h_k\in\mathbb{C}$ is the channel coefficient between user $k$ and the server, which is assumed to follow the Rayleigh distribution, i.e., $h_k\sim CN(0, \sqrt{\lambda})$. $n\sim N(0, \sigma_z^2)$ is the additive white Gaussian noise (AWGN) at the receiver. The Tx-scaling factor $p_k\in\mathbb{C}$, a.k.a. power control policy, compensates the phase shift posed by the channel and amplifies the transmit signal. The goal of the Tx-scaling is to ensure that each participating user contributes equally at the receiving antenna and the superposed signal is proportional to the ideal summation, which is defined as the average operation over the input set without AirComp, i.e., $y:= \frac{1}{K}\sum_{k=1}^{K}x_k$. Accordingly, the Rx-scaling factor $a\in\mathbb{R}$ acts as an equalizer and recovers the sampled analog result to its expected value. 

%digital domain. 

\subsection{Preliminary Experiments on AirComp FL}\label{motiv_aircompfl}

\begin{figure} \centering %\hspace{-3em}
 \subfigure[Test accuracy vs. epochs\label{diss_comm}]
  {\includegraphics[width=.48\textwidth]{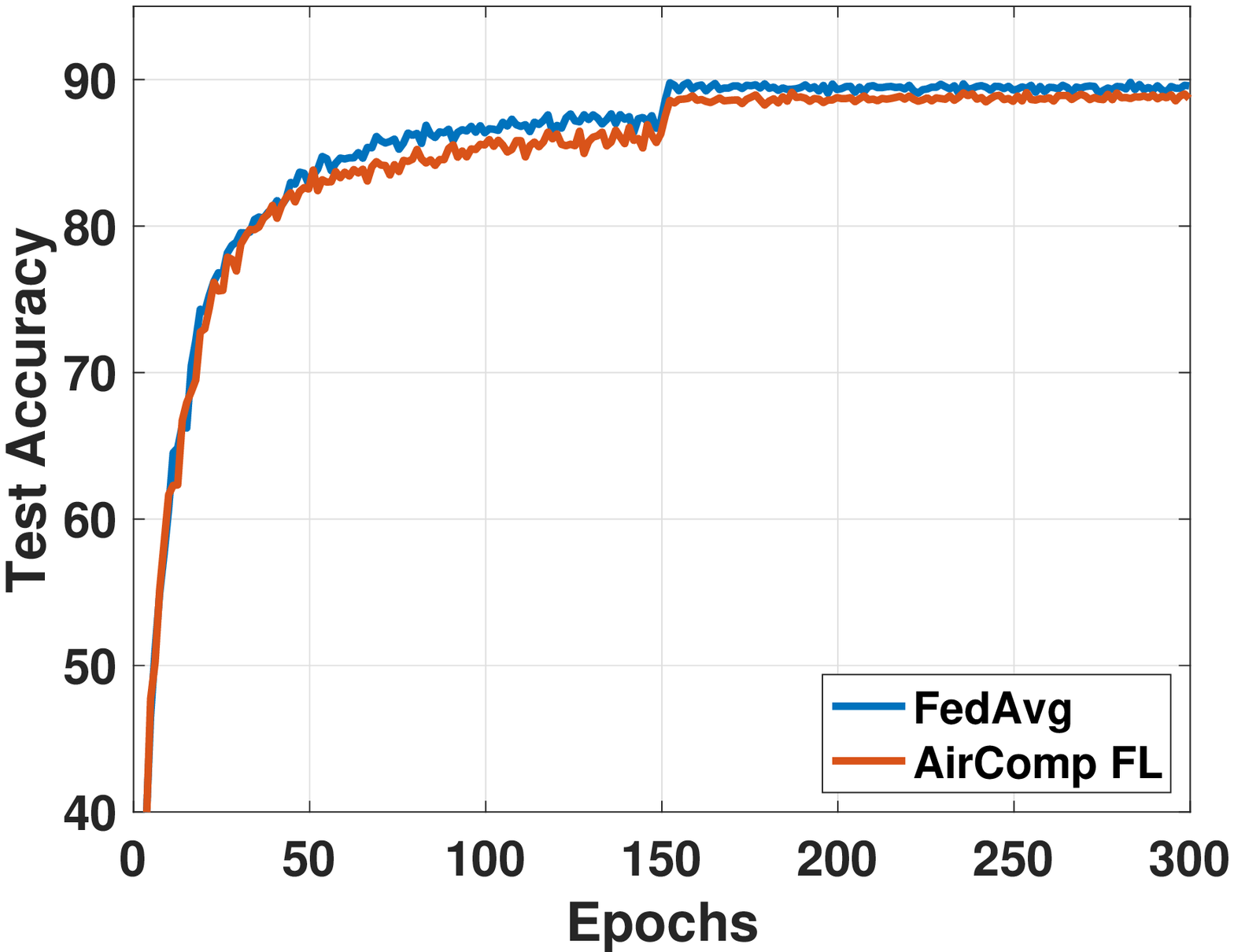}}
  \subfigure[Test accuracy vs. comm. resources\label{diss_latency}]
  {\includegraphics[width=.48\textwidth]{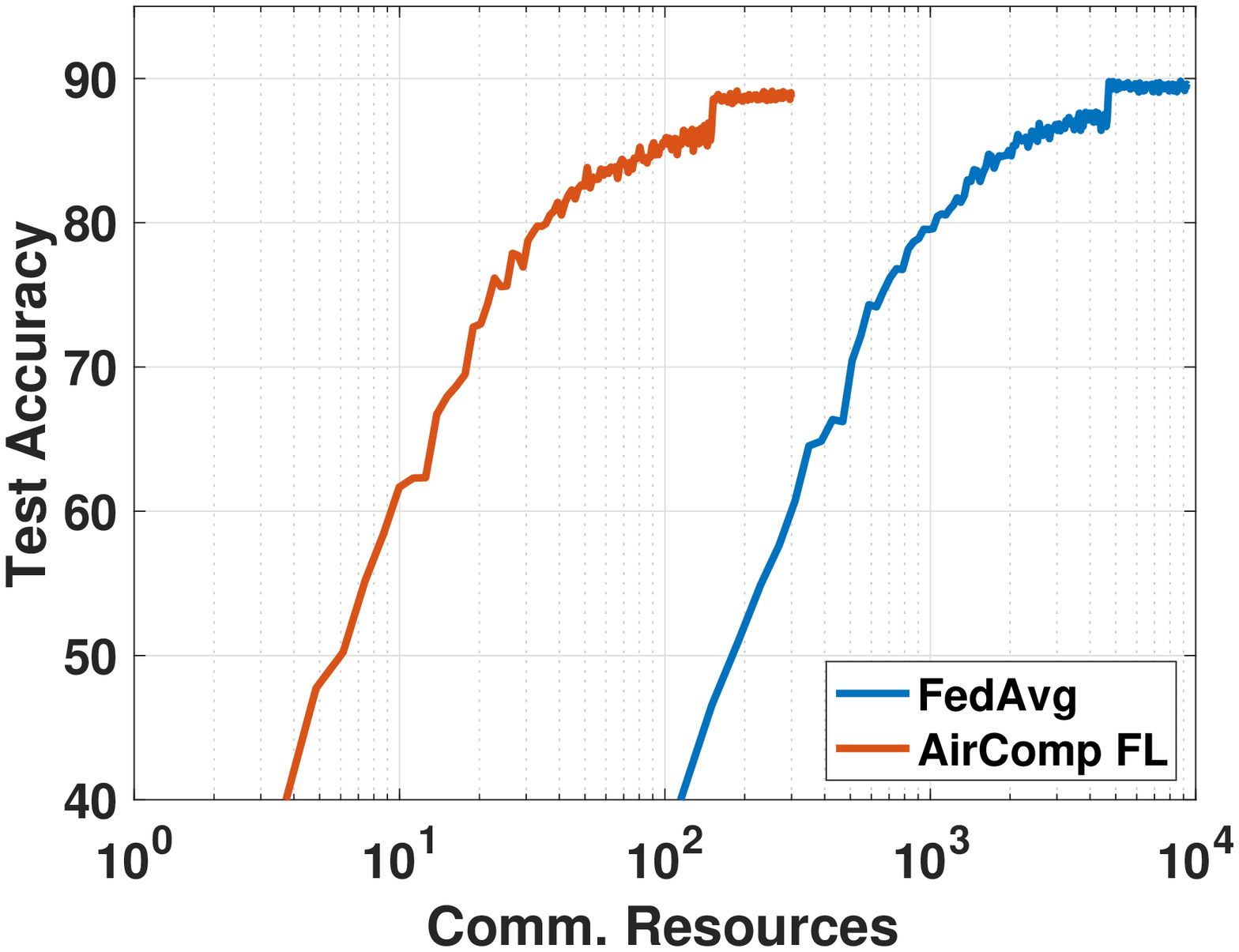}}
  \caption{Over-the-Air federated learning (AirComp FL).} \label{Fig.diss}
\end{figure}
To demonstrate the spectrum-efficient benefit of AirComp FL, we present the preliminary experimental results AirComp FL and the classic FedAvg in Fig.~\ref{Fig.diss}, which correspond to the FL scheme with and without AirComp respectively. Here, $10$ users are considered to participate in an FL task and collaboratively train a ResNet-20 model on the CIFAR-10 dataset. Both the communication bandwidth and the number of training epochs are set to be the same for these two schemes. Taking test accuracy as the measure, Fig.~\ref{diss_comm} depicts the convergence performance of the training process, while Fig.~\ref{diss_latency} displays the communication resource consumption during the training. It shows that, compared with FedAvg, AirComp FL only requires a little more or even the same number of data epochs to achieve the target test accuracy, thereby imposing negligible impacts on convergence rate. Meanwhile, the communication resource consumption of the AirComp FL is much less than that of FedAvg, since the latter forces the users to use orthogonal channels for interference avoidance instead of performing concurrent transmission over the same spectrum. Note here that we use the normalized communication resources for Fig.~\ref{diss_latency} illustration and assume one unit communication resource is consumed in each communication round in AirComp FL.

%For the communication spectrum, the users cannot perform the concurrent transmission with the same bandwidth in FedAvg. Therefore, . which demonstrates that AirComp FL can significantly improve spectrum efficiency compared with FedAvg. Here, we assume that AirComp FL consumes a unit communication resource in each communication round.

\section{The Design of M-AirComp and M-AirComp-based FL}

\subsection{M-AirComp Design}

Different from the most existing AirComp approaches with an analogy modulation scheme, we establish a digital modulation scheme for the AirComp to cater for the commercial transmit devices and design a multi-bit over-the-Air computation scheme (M-AirComp). To this end, the Rx-scaling factor $a$ performs as a digital domain equalizer, and the division operation in Eq.~\ref{eq:AirCompAvg} to calculate the arithmetic average is also in the digital domain. In order to eliminate the burden of redesigning the modulation scheme, we tend to integrate the gradient quantization to the most common Quadrature Amplitude Modulation (QAM) in LTE, 5G, and Wi-Fi 6 standard~\cite{x3gpp2014}. Instead of transmitting arbitrary values, gradients to transmit are clipped and quantized as Multiple Amplitude Shift Keying (MASK) symbols, so as to be compatible with modern digital devices. Two MASK-modulated gradients can be transmitted orthogonally using in-phase (I) and quadrature (Q) channel simultaneously. We notice that it is equivalent to map two separate gradients onto a symbol from the square $M^2$ QAM constellation. Here, we limit $M$ between $2$ to $2^b$. For example, when $b$ is set as 3, the user will use 64QAM to transmit two gradients, as shown in Fig.~\ref{fig:MOTA}. In this way, altering the value $M$ at the transmitter allows full digital data transmission while preserving $b$-bit resolution, according to the estimated channel gain. 
%The edge devices usually have limited computation resources and power supply.

\begin{figure}
    \centering
    \includegraphics[width = 0.7\linewidth]{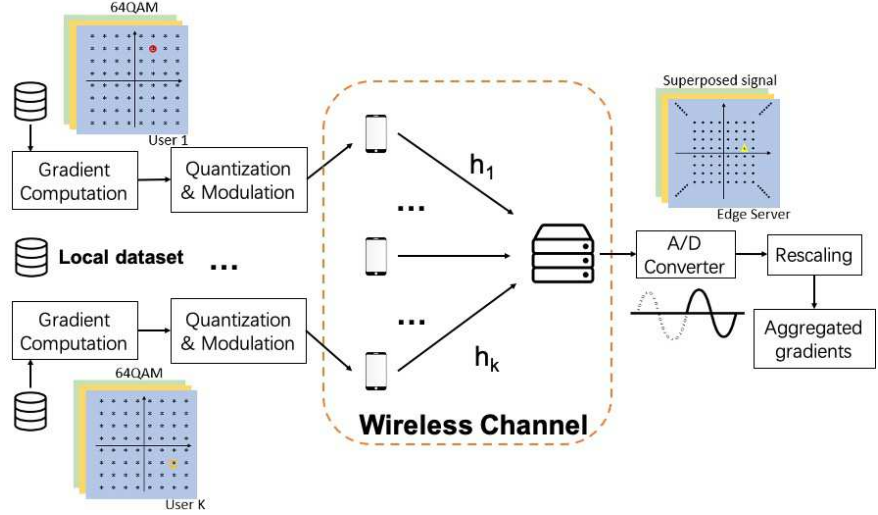}
    \caption{Multi-bit Over-the-Air computation design.}
    \label{fig:MOTA}
\end{figure}

Assume that the server equips with a high-resolution analog-to-digital converter (ADC) (e.g., 16-bit). While receiving, multiple QAM symbols superpose at the sampling instance, which can be viewed from (a part of) a higher-order rectangular QAM constellation diagram (when the number of mobile devices is odd) or a zero-centered constellation diagram (when the number of users is even). Since the biggest possible value after aggregation can be obtained from user feedback, we can utilize this value as the ADC reference voltage. In order to alleviate the detection complexity, we directly use the quantized samples followed by Rx-scaling defined in Eq.~\ref{eq:AirCompAvg} in the digital domain. In this way, the transmission module is implemented in a digital manner, which enables the M-AirComp to have better compatibility compared with traditional AirComp. The process is also illustrated in Fig.~\ref{fig:MOTA}. This result can be viewed as the desired computational result added by quantization error and channel noise, whose impacts on federated learning performance are analyzed in the following section.

During the transmission process, each device is constrained by an average transmitting power budget $P^0$. Assume that all the participating devices have the same power budget, where the average transmission power constraint is given by:
\begin{align}\label{power_cons}
    \mathbb{E}[|p_k|^2] \leq P^0, \forall k.
\end{align}

Due to the power limit, some users facing severe signal fading may not completely align their amplitude, which means that the Tx-scaling factor $p_k\in\mathbb{C}$ cannot be infinitely enlarged to meet the amplitude alignment requirement. Therefore, we adopt an energy efficient power control policy that the users with poor channel conditions are not allowed to transmit, i.e. the transmit power is set to be zero. Let $g_{\text{th}}$ be the channel gain threshold of possible transmission, and the power control policy $p_k$ can be represented as:
\begin{align}
p_{k}=\left\{\begin{array}{ll}
 \frac{\sqrt{\varrho}h_{k}^{\dagger}}{{\left|h_{k}\right|}^2}, & \left|h_{k}\right|^{2} \geq g_{\text {th }} \\
0, & \left|h_{k}\right|^{2}<g_{\text {th }}.
\end{array}\right.
\end{align}

Here, $\varrho$ is a scaling factor to guarantee the desired SNR. Under the above power control policy, only users facing channel gain larger than  $g_{\text {th }}$ can be allowed to transmit their updated gradients. Note that the threshold $g_{\text {th }}$ can be adjusted to control the gradient transmission. Due to the power constraint in Eq.~\ref{power_cons}, the threshold $g_{\text{th}}$ can be set as an arbitrary value larger than a minimum value $g_{\text{th}}^{\text{min}}:=h^2=\frac{\varrho}{p^2}=\frac{\varrho}{P^0}$. Specifically, under a certain communication environment, the greater the threshold $g_{\text{th}}$ we set, the larger the number of allowable transmitting gradients. By varying the threshold $g_{\text{th}}$, our M-AirComp design has the potential to only involve gradients with good channel conditions, which require less power to transmit and thereby benefits in energy saving. We define a long-term average transmission probability $p_b$ to indicate the degree of gradient participation, where any threshold $g_{\text{th}}$ will correspond to a transmission probability. Assume that the channel coefficient is Rayleigh distributed, i.e., $h_k\sim CN(0, \sqrt{\lambda})$ and the channel gain follows an exponential distribution. The transmission probability $p_b$ corresponding to the threshold $g_{\text {th }}$ can be calculated as:
\begin{align}
\begin{split}
    p_b = \int_{g_\text{th}}^{\infty}\lambda e^{-\lambda x}\mathrm{d} x = e^{-\lambda g_{\text {th }}}.
\end{split}
\end{align}

If the probability of keeping these gradient elements to transmit is $p_b$, the Rx-scaling factor $a$ will be set as $\frac{1}{\sqrt{\varrho}p_b}$ to rescale the received signal. Due to the property of the Rayleigh fading channel and the power constrain of the local user devices, the achievable highest transmission probability $p_b^\text{max}$ is calculated as $p_b^\text{max} = e^{-\lambda g^\text{min}_{\text {th }}}=e^{-\lambda \frac{\varrho}{P_0}}$.

\begin{algorithm}[t]
    \caption{\textbf{E}nergy and \textbf{S}pectrum Efficient \textbf{O}ver the \textbf{A}ir \textbf{F}ederated \textbf{L}earning Algorithm (ESOAFL)} \label{Alg1}
\hspace*{0.02in} {\bf Initialization:} Initialize the global model $\mathbf{w}^0$ and set $\mathbf{w}^{0,0}_k=\mathbf{w^0}, \forall k \in \mathcal{K}$; Set the learning rate $\gamma$ and $\eta$, local computing iterations $H$, and the channel gain threshold $g_\text{th}$\\
\hspace*{0.02in} Initialize the communication index $r=0$ and the local computing iteration count $h=0$
\begin{algorithmic}[1]
\While{$r < R$}
\For{$h=0,...,H-1$}
\State Each device $k$ computes the unbiased stochastic gradients $\nabla F_k(\mathbf{w}^{r,h}_k)$ of $f_k(\mathbf{w}_k^{r})$ with one batch size of data from the dataset $\mathcal{D}_k$
\State Each device $k$ in parallel updates its local model:
%\begin{align}
    $\mathbf{w}^{r,h+1}_k = \mathbf{w}^{r,h}_k - \eta \nabla F_k(\mathbf{w}^{r,h}_k), \quad \forall k\nonumber$
%\end{align}
\EndFor
\State \textbf{end for}
\State Each device $k$ calculates the accumulated gradients with gradient quantization as $Q\left(\eta\sum_{h=0}^{H-1} \nabla F_k(\mathbf{w}^{r,h}_k)\right)$
\State Each device $k$ transmits the quantized accumulated gradients if the observed channel gain larger than the pre-selected threshold $g_{\text {th }}$, i.e., $\left|h_{k}\right|^{2} \geq g_{\text {th }}$; otherwise, no transmission
\State All transmitted gradients are aggregated over the air and the global model is updated as in Eq.~(\ref{glo_update}) 
\State Update $r\leftarrow r+1$
\State Each device $k$ updates its local model $\mathbf{w}_k^{r,0}=\mathbf{w}^{r}$
\EndWhile
\State \textbf{end while}
\end{algorithmic}
\end{algorithm}

\subsection{M-AirComp-based FL Design}

Based on M-AirComp, this subsection presents an \textbf{E}nergy and \textbf{S}pectrum Efficient \textbf{O}ver the \textbf{A}ir \textbf{F}ederated \textbf{L}earning (\textbf{ESOAFL}) algorithm integrating gradient quantization, as shown in Fig.~\ref{fig:MOTA}, where the goal is to improve the spectrum efficiency while reducing the energy consumption of the participating devices. The pseudocode of our algorithm is given in Alg.~\ref{Alg1}, and the details are described in the following.

Following the ESOAFL, all the participants start the training procedure with the initialized model parameters. Specifically, every FL user executes $H$ local computing SGD steps with mini-batch size data drawn from their own datasets. After the local training, a uniform gradient quantization operator $Q(\cdot)$ is utilized to quantize the updated gradients into low bits, i.e., 4-bit or 8-bit. Taking $b$-bit quantization as an example, the local updates of all the participants are quantized to $2^b$ levels with a specific maximum/minimum value, catering to the digital wireless transmission scheme. Next, for the transmission process, every $2$ gradient element is modulated into one digital symbol over the sub-channel according to our M-AirComp design. We assume the symbol-level synchronization among all the mobile devices that ensures coherent and concurrent transmission. This assumption can be realized by dedicating the bandwidth for mobile device synchronization, e.g., 1.08 MHz primary synchronization channel (PSCH) and secondary synchronization channel (SSCH) in LTE system \cite{sriharsha2017complete}, or the AirShare\cite{abari2015airshare} for distributed MIMO synchronization. Then we employ the M-AirComp operator $\text{Air}(\cdot)$, along with the proposed energy efficient power control policy. The threshold $g_{\text {th }}$ is determined firstly, and then gradient element whose corresponding channel gain larger than this threshold can be allowed to transmit. In this way, the long-term transmission probability $p_b$ can be calculated as $e^{-\lambda g_{\text{th}}}$. Because M-AirComp integrates wireless transmissions and aggregation over the air, the server receives only the aggregated updated gradients. Finally, the server updates the global model with the aggregated updated gradients, which can be represented as:
\begin{align}\label{glo_update}
    \mathbf{w}^{r+1} = \mathbf{w}^{r} - \text{Air}\left(\left\{ Q\left(\eta\sum_{h=0}^{H-1} \nabla F_k(\mathbf{w}^{r,h}_k)\right)\right\}_{\mathcal{K}}\right).
\end{align}

After updating the global model, the server will broadcast the global model to all devices for continuing training. We repeat the above procedure for $R$ rounds until the model converges to a stationary point. Particularly, the convergence requirement can be represented as $\frac{1}{R}\sum_{r=0}^{R-1}\|\nabla f^r\|^2_2 \leq \epsilon$, where $\epsilon$ denotes the target training loss and $\nabla f^r$ is the global function gradient at the $r$-th communication round. 

\section{Spectrum and Energy Efficient FL: Formulation and Solutions}

In this section, we formulate an overall energy minimization problem and establish the communication and computation energy models of the proposed ESOAFL algorithm. Based on the derived convergence analysis, we then optimize the control policy in terms of the transmission probability $p_b$ and local computing iterations $H$ to minimize the overall energy consumption.

\subsection{Energy Minimization Problem Formulation}

It is challenging to deploy energy-hungry FL tasks on mobile devices due to their limited battery capacity. Hence, in this work, we aim to minimize the total energy consumption of FL training via joint control of local computing iterations $H$ and transmission probability $p_b$. The average energy consumption per communication round of mobile device is cast as $E = E^{comm}(p_b) +  E^{comp} H$. Here, $E^{comm}(p_b)$ is the communication energy consumed to transmit the updated gradients, which is related to the transmission probability $p_b$, and $E^{comp}$ is the computing energy of performing one local iteration. The goal is to minimize the overall energy consumption during the federated training while guaranteeing the model convergence, denoted as 
\begin{align}
\begin{split}
     \min \quad &\mathbb{E}\left[E_{tot}\right] \triangleq \mathbb{E}\left[RE^{comm}(p_b)\right]  +   \mathbb{E}\left[RE^{comp}H\right]\\
     \text{s.t.,} \quad &\frac{1}{R} \sum_{r=0}^{R-1}  \mathbb{E}\left[\|\nabla f^r\|^2_2\right] \leq \epsilon.
\end{split}
\end{align}

\subsection{Communication and Computation Energy Models}

\subsubsection{Communication model} 
If we consider the M-AirComp power control policy with transmission probability $p_b$ whose value is smaller than $p^{\text{max}}_b$, the threshold channel gain is mapped as $g_\text{th}:= -\frac{1}{\lambda}\ln {p_b}$. In this way, the average power consumption among all the users and time slots will be:
\begin{align}
\begin{split}
    P^{\textit{comm}} =& p_b\varrho\int_{g_\text{th}}^{\infty} \lambda \frac{1}{x} e ^{-\lambda x} \mathrm{d} x\\
    =& -p_b\varrho\lambda\text{Ei}\left(-\lambda g_\text{th}\right)= -p_b\varrho\lambda\text{Ei}\left(\ln {p_b}\right),
\end{split}
\end{align}
where $\text{Ei}(x)$ is the exponential integral function denoted as $\text{Ei}(x) = \int_{-\infty}^{x}\frac{e^t}{t}\mathrm{d} x$. Due to the fact that $-\ln {p_b}$ is positive, we have $\text{Ei}\left(\ln {p_b}\right) = - \text{E}_1\left(- \ln {p_b}\right)$ where $\text{E}_1(x) = \int_{x}^{\infty}\frac{e^{-t}}{t}\mathrm{d} x$. Then we have $P^{comm} = -p_b\varrho\lambda\text{Ei}\left(\ln {p_b}\right) = p_b\varrho\lambda\text{E}_1\left(- \ln {p_b}\right)$. For any $x$ with positive real value, $\text{E}_1(x)$ can be bracketed by elementary functions as follows:
\begin{align}
    \text{E}_1(x) < e^{-x} \ln \left(1 + \frac{1}{x} \right).
\end{align}
Due to $-\ln {p_b} > 0$, we have 
\vspace{-1mm}
\begin{align}
    P^{comm} \approx p_b\varrho\lambda e^{\ln {p_b}} \ln \left(1 + \frac{1}{- \ln {p_b}} \right) = \varrho\lambda {p_b}^2 \ln(1 - \frac{1}{\ln {p_b}}).
\end{align}

%The first and second order derivatives of the $P^{comm}$ can be calculated as:
%\begin{align}
%\begin{split}
%    &\frac{\partial P^{comm} }{\partial p} = \lambda \ln\left(1- \frac{1}{\ln p}\right) + \lambda \frac{1}{\ln p(\ln p-1)}\\
%    &\frac{\partial^2 P^{comm} }{\partial^2 p} = \lambda \frac{\ln^2p - 3 \ln p+1}{p\ln^2 p(\ln p-1)^2}
%\end{split}
%\end{align}

After receiving the full precision gradients, each device is required to quantize the gradient into low-bit precision for digital transmission. Then, we adapt MASK to modulate the gradients, which means the magnitude of each symbol is sufficient to decode the transmission gradient. Let $T_s$ denote the symbol duration that is in inverse proportion to sub-channel bandwidth $B$. To transmit the  updated model with the size of $d$ gradients, $d/2$ symbol is required according to the M-AirComp design. Thus, the transmission time can be represented as $T^{comm} = \frac{d}{2M_s}T_s$, where $M_s$ symbols are transmitted in parallel.
%\begin{align}
%    T^{comm} = \frac{d}{M}T_s,
%\end{align}
%where $M$ symbols are transmitted in parallel. 

Accordingly, the communication energy consumption for each device in each communication round is the product of the average transmission power and the transmission time, which is computed as
\begin{align}
    E^{comm} = P^{comm}\times T^{comm}.
\end{align}

%\subsection{Computational Energy Model}
\subsubsection{Computational model} 
With massive data generated or collected on mobile devices, local on-device computing can naturally be treated as computation-hungry tasks. Luckily, most modern smart devices are equipped with high-performance GPUs and can handle such heavy training tasks efficiently. This works considers the GPU computational energy model. We model the energy consumed to process a mini-batch of data in one iteration as a product of the runtime power and the execution time, i.e.,
\begin{equation}\label{ite_ener}
    E^{comp}= P^{comp} \times T^{comp},
\end{equation}
where $P^{comp}$ and $T^{comp}$ are runtime power and execution time of the edge device, respectively. Both of them are related to the GPU core frequency/voltage and the memory frequency in the forms of~\cite{mei2017energy}
\begin{equation}
    P^{comp}=P^0+a f^{mem}+ b (v^{core})^2 f^{core},
\end{equation}
\begin{equation}
    T^{comp}=T^0+\frac{u}{f^{mem}} + \frac{v}{f^{core}}.
\end{equation}

$P_0$ and $T_0$ are the static power and static time consumption; $f^{core}/v^{core}$ and $f^{mem}$ represent the core frequency/voltage and memory frequency, respectively. $a$, $b$, $u$, and $v$ are constant coefficients that reflect the sensitivity of the task execution to GPU memory and core frequency/voltage scaling~\cite{mei2017energy,abe2014power}. Given a specific FL task, i.e., a neural network model and the corresponding dataset, such coefficients can be accurately estimated based on the experiments. Since there are $H$ local computing iterations between two sequential communication rounds, the energy consumption of local computing in one communication round can be calculated as the product of the energy consumption of one iteration and the number of local iterations, i.e., $E^{comp} \times H$.

\subsection{Impacts of Control Variables on ESOAFL Convergence}
In this subsection, we derive the convergence analysis of the ESOAFL approach, where we theoretically analyze the impacts of control variables $p_b$ and $H$ on training convergence. Firstly, we have the following model assumptions.
\vspace{-0.1cm}
\begin{assumption}[Smoothness] The objective function $f_i$ is differentiable and L-smooth :
\begin{align}
    \|\nabla f_k(\mathbf{x})-\nabla f_k(\mathbf{y})\| \leq L\|\mathbf{x}-\mathbf{y}\|, \forall k.
\end{align}
\end{assumption}

\begin{assumption}[Bounded variances and second moments]\label{bound}

The variance and the second moments of stochastic gradients evaluated with a mini-batch can be bounded as
\begin{align}
    \mathbb{E}_{\xi_{i} \sim \mathcal{D}_{i}}\left\|\nabla F_{i}\left(\mathbf{w} ; \xi_{i}\right)-\nabla f(\mathbf{w})\right\|^{2} \leq \sigma^{2}, \forall \mathbf{w}, \forall i,\\
    \mathbb{E}_{\xi_{i} \sim \mathcal{D}_{i}}\left\|\nabla F_{i}\left(\mathbf{w} ; \xi_{i}\right)\right\|^{2} \leq \delta^{2}, \forall \mathbf{w}, \forall i,
\end{align}
where $\sigma$ and $\delta$ are positive constants.
\end{assumption}

\begin{assumption}[Quantization bounded variances]
The output of the quantization operator $Q(x)$ is an unbiased estimator of its input $x$, and its variance grows with the squared of L2-norm of its argument, i.e., $\mathbb{E}[Q(x)]=x$ and  $\mathbb{E}[||Q(x)-x||^2]=q||x||^2$.
\end{assumption}

Based on the above assumptions, we have the following lemma on the bounded variances of M-AirComp, where the power control policy with a transmission probability $p_b$ is applied for gradient uploading.

\begin{lemma}[M-AirComp bounded variances]
The output of the M-AirComp operator $\text{Air}(\mathcal{X})$ with the proposed power control scheme is an unbiased estimator of its input set $\mathcal{X}$, and its variance decreases with the increasing of the transmission probability and grows with the squared of its argument, i.e., $\mathbb{E}[\text{Air}(\mathcal{X})]=y$ and  $\mathbb{E}[||\text{Air}(\mathcal{X})-y||^2]=\frac{1}{K^2}(\frac{1}{p_b}-1)\sum_{x_k \in \mathcal{X}}x_k^2 + \frac{\sigma_z^2}{K^2{p_b}^2}$.
\end{lemma}

\begin{proof}

Let $\mathcal{X}$ be the input set of the M-AirComp operator, and we further define $\bar{\mathcal{X}}$ as the successful transmit set to help the proof. Accordingly, the mean and the mean of the square values can be expressed as:
\begin{flalign}
     \mathbb{E}[\text{Air}(\mathcal{X})]
     =&\mathbb{E}\left[\frac{1}{p_b K}\Bigg[\sum_{x_k \in \bar{\mathcal{X}}}x_{k} +\sum_{x_k \notin \bar{\mathcal{X}}}x_{k} + n\Bigg]\right]\nonumber\\
     =&\frac{1}{p_bK}\left[\sum_{x_k \in \bar{\mathcal{X}}}x_k\cdot p_b + \sum_{x_k \notin \bar{\mathcal{X}}}0\cdot (1-p_b) + \mathbb{E}\left[ n\right]\right]=y,\\
%\end{flalign}
%We further calculate the mean of the square value as:
%\begin{flalign}
    \mathbb{E}[(\text{Air}(\mathcal{X}))^2]
     =&\mathbb{E}\left[\frac{1}{{p_b}^2K^2}\left(\sum_{x_k \in \mathcal{X}}x_{k} + n\right)^2\right]\nonumber\\
     =&\mathbb{E}\left[\frac{1}{{p_b}^2K^2}\left(\sum_{x_i \in \mathcal{X}}\sum_{x_j \in \mathcal{X}}x_{i}x_{j} + 2\sum_{x_k \in \mathcal{X}}x_{k} n +  n^2\right)\right]\nonumber\\
     =&\frac{1}{{p_b}^2K^2}\left[\sum_{x_i \in \mathcal{X}}\sum_{x_j \in \mathcal{X}, i \neq j} x_i{p_b}x_j{p_b} + \sum_{x_k \in \mathcal{X}}x_k^2 {p_b}\right] + \frac{\sigma_z^2}{K^2{p_b}^2}\nonumber\\
     =&\frac{1}{{p_b}^2K^2}\left[ {p_b}^2\left((\sum_{x_k \in \mathcal{X}}x_k)^2-\sum_{x_k \in \mathcal{X}}x_k^2 \right) + {p_b} \sum_{x_k \in \mathcal{X}}x_k^2 + {\sigma_z^2}\right] \nonumber\\
     =&\frac{1}{K^2}\left((\sum_{x_k \in \mathcal{X}}x_k)^2 + (\frac{1}{{p_b}}-1)\sum_{x_k \in \mathcal{X}}x_k^2 \right) + \frac{\sigma_z^2}{K^2{p_b}^2}.
\end{flalign}

Thus, the variance is equal to the mean of the square value minus the square of the mean value, which is represented as:

\begin{align}\label{AirComp_variance}
\begin{split}
    \text{Var}(\text{Air}(\mathcal{X})) =& \mathbb{E}[(\text{Air}(\mathcal{X}))^2] - \mathbb{E}[\text{Air}^2(\mathcal{X})]\\
    =& y^2 + \frac{1}{K^2}(\frac{1}{{p_b}}-1)\sum_{x_k \in \mathcal{X}}x_k^2 + \frac{\sigma_z^2}{K^2{p_b}^2} - y^2\\
    =&\frac{1}{K^2}(\frac{1}{{p_b}}-1)\sum_{x_k \in \mathcal{X}}x_k^2 + \frac{\sigma_z^2}{K^2{p_b}^2}.
\end{split}
\end{align}

\end{proof}

\begin{theorem}\label{theorem_1}

For the proposed ESOAFL approach, under the above assumptions, if learning rates $\theta$ and $\eta$ satisfy
\begin{align}
    1\geq L^2\eta^2H^2 + HL\theta \eta\frac{q(2-{p_b}) +K{p_b}}{K{p_b}},
\end{align}
the convergence rate after $R$ communication rounds can be bounded as: 
\begin{align}
\begin{split}
    \frac{1}{R}\sum_{r=0}^{R-1}\|\nabla f^r\|^2_2 \leq \frac{2 (f(\mathbf{w}^{0})-f(\mathbf{w}^{*})}{\eta\theta  HR}+\frac{\eta\theta  L}{K}\frac{(p_b+q)}{p_b}\sigma^2+ \eta^2L^2H\sigma^2 + \frac{\theta \eta L}{HK^2{p_b}^2}\sigma_z^2,
\end{split}
\end{align}
where $q$ is the gradient quantization precision, $p_b$ is the M-AirComp transmission probability, $H$ is the local computing iterations, and $f(\mathbf{w}^{*})$ is the minimum value of the loss.
\end{theorem}

\begin{proof}
    Please refer to the Appendix.~A for the proof. 
\end{proof}

The above Theorem~\ref{theorem_1} is derived based on the L-smoothness gradient assumption on global objective~\cite{haddadpour2021federated}. After expanding the inequality of the global objective, we first bound the inner product between the stochastic gradient and full batch gradient, while we can also bound the distance between the global model and the local model. Further, we bound the updated gradients with M-AirComp and quantization operators. Finally, by integrating the derived results above, we finish the convergence analysis of the ESOAFL algorithm.

%Please refer to Appendix\footnote{Appendix is available at https://github.com/nonymity/a/blob/main/1.pdf} for the proof.

\begin{corollary}\label{conv_coroll1}

To achieve the linear speedup, we need to have $\theta  \eta = O\left(\frac{\sqrt{K}}{\sqrt{RH}}\right)$. If we further choose $\theta  \eta =O\left(\frac{1}{L}\sqrt{\frac{K{p_b}}{RH({p_b}+q)}}\right)$, the convergence rate can be represented as:

\begin{flalign}
     &\frac{1}{R}\sum_{r=0}^{R-1}\|\nabla f^r\|^2_2  \leq \frac{2L (f(\mathbf{w}^{0})-f(\mathbf{w}^{*})\sqrt{({p_b}+q)}}{\sqrt{KRH{p_b}}} +\\
     &\frac{\sqrt{{p_b}+q}}{\sqrt{KRH{p_b}}}\sigma^2  \nonumber + \frac{K}{R\theta ^2}\sigma^2 + \sqrt{\frac{1}{K^3RH^3({p_b}+q){p_b}^3}}\sigma_z^2\nonumber\\
     &\overset{(a)}=O\left( \frac{\sqrt{{p_b}+q}}{\sqrt{KRH{p_b}}}(2L(f(\mathbf{w}^{0})-f(\mathbf{w}^{*}) + \sigma^2)) + \frac{K}{R\theta ^2}\sigma^2   \right)\nonumber\\
     &\overset{(b)}=O\left(\frac{\chi }{\sqrt{KRH}}\right) + O\left(\frac{K}{R}\right),\nonumber
\end{flalign}
where $(a)$ is due to the fact that $O(\sqrt{\frac{1}{K^3R}})$ decays faster than $O(\sqrt{\frac{1}{KR}})$, and we replace $ \sqrt{\frac{{p_b}+q}{{p_b}}}$ by $\chi$ in $(b)$.
\end{corollary}

%After establishing the communication and computing energy models, another key component to formulate the overall energy consumption problem is to obtain the required communication rounds. Accordingly, we can obtain it from the derived convergence analysis.
Based on the convergence analysis, we further give the following corollary on the communication complexity, i.e., the required number of communication rounds, of our ESOAFL algorithm.

\begin{corollary}\label{conv_coroll2}
From the Corollary~\ref{conv_coroll1}, the required maximum number of communications for achieving the $\epsilon$ target training loss, i.e., satisfying $\epsilon = \frac{1}{R}\sum_{r=0}^{R-1}\|\nabla f^r\|^2_2$, is given by
\begin{flalign}\label{big_O}
    R &= O\left(\frac{2 \epsilon \sigma^2 H K^{2}+\chi^{2} (\delta + \sigma^2)^{2} \theta ^{2}}{2 \epsilon^{2} \theta ^{2} H K}\right)+O\left(\frac{+\chi (\delta + \sigma^2) \theta  \sqrt{4 \epsilon \sigma^2 H K^{2}+\chi^{2} (\delta + \sigma^2)^{2} \theta ^{2}}}{2 \epsilon^{2} \theta ^{2} H K}\right)\nonumber\\
    &= O\left(K\right)+ O\left(\frac{\chi^2}{HK}\right) + O\left(\frac{\chi}{\sqrt{H}}\right),
\end{flalign}
where $\chi = \sqrt{\frac{p_b+q}{p_b}}$ and $\delta = 2L (f(\mathbf{w}^{0})-f(\mathbf{w}^{*}))$.

\end{corollary}

\subsection{Overall Energy Minimization Reformulation and Solution}
With the above models, we calculate the total energy consumed by the participating mobile devices during the entire training process as:

\begin{subequations}  
\begin{align}
E_{total}(p_0,H)  &= R \times \left(E^{comm} + H E^{comp} \right)\label{OrigObj} \\
& =\left(\frac{A_0(p_b+q)}{p_bH} + \frac{B_0\sqrt{p_b+q}}{\sqrt{p_b H}}+ C_0\right)\times \left(\varrho\lambda {p_b}^2 \ln(1 - \frac{1}{\ln p_b})T^{comm} + H E^{comp} \right), \nonumber
\end{align}
\end{subequations}
where $A_0$, $B_0$, and $C_0$ are constants used to approximate the big-$O$ notion in Eq.~\ref{big_O}.
From the above formula, we observe that increasing the local computing iterations $H$ reduces the needed communication rounds $R$ (``talking"), but increases the computing energy consumption per round (``working"). Similarly, adjusting $p_b$ also affects the required communication rounds and the energy consumption of each round. Thus, it is necessary to optimize $H$ and $p_b$ to balance the ``working" and ``talking", thus minimizing the overall energy consumption. To this end, we formulate the Joint local Computing and transmission Probability (JCP) control problem as:
\begin{subequations}  \label{OrigProb}
\begin{align}
\min_{\substack{p_b,H}} \quad  &\left(\frac{A_0(p_b+q)}{p_bH} + \frac{B_0\sqrt{p_b+q}}{\sqrt{p_b H}}+ C_0\right)\times \left(\varrho\lambda {p_b}^2 \ln(1 - \frac{1}{\ln p_b})T^{comm} + H E^{comp} \right)\\
s.t.\quad & 0 < p_b \leq p^{max}_b, \label{OrigC1}\\
    \quad & H \in \mathcal{H}, \label{OrigC2}
\end{align}
\end{subequations}

For notational simplicity, we define $\boldsymbol{\phi}=\{p_b,H\}$ and represent the objective function as \textbf{$\Theta(\boldsymbol{\phi}) = \Theta _1(\boldsymbol{\phi})\times \Theta _2(\boldsymbol{\phi})$}, where
\begin{flalign} 
    &\Theta _1(\boldsymbol{\phi}) = \frac{A_0(p_b+q)}{p_b H} + \frac{B_0\sqrt{p_b+q}}{\sqrt{p_bH}}+ C_0, \\
    &\Theta _2(\boldsymbol{\phi}) = \varrho\lambda {p_b}^2 \ln(1 - \frac{1}{\ln p_b})T^{comm} + H E^{comp}.
\end{flalign} 

Noticing the decoupled constraints in (\ref{OrigC1}-\ref{OrigC2}), we relax the constraint in (\ref{OrigC2}) as $H_{min} \leq H \leq H_{max}$, where $H_{min}$ and $H_{max}$ are the minimum and the maximum integer in $\mathcal{H}$, respectively. Moreover, we can identify that both function $\Theta _1(\boldsymbol{\phi})$ and $\Theta _2(\boldsymbol{\phi})$ are positive and convex after calculating the first and second-order partial derivative of these two functions. (Please refer to Appendix.~B for the detailed derivation.)

\begin{algorithm}[!t]
\caption{JCP Control Algorithm} \label{JCP}
\hspace*{0.02in} {\bf Initialization:}  $\epsilon, \xi, \iota = 10^{-5}$; $\gamma^0\in(0,1]$; $\kappa=0$
\begin{algorithmic}[1]
\Repeat
    \State Solve (\ref{ApproProb}) and set the optimal value as $\boldsymbol{\phi}^*(\boldsymbol{\phi}^\kappa$) 
    \State Set $\boldsymbol{\phi}^{\kappa+1}=\boldsymbol{\phi}^\kappa+\gamma^0(\boldsymbol{\phi}^*(\boldsymbol{\phi}^\kappa)-\boldsymbol{\phi}^\kappa)$
    \State Set $\kappa=\kappa+1$ and $\gamma^\kappa=\gamma^{\kappa-1}(1-\xi \gamma^{\kappa-1})$
    \Until{$||\boldsymbol{\phi}^\kappa-\boldsymbol{\phi}^{\kappa-1}||_2^2\leq \iota$}
    \State Round the current $H$ to the  nearest integer in $\mathcal{H}$
\State \Return The current solutions of $p_b$ and $H$.
\end{algorithmic}
\end{algorithm}

Capturing such the ``product-of-convexity'' property of the objective function $\Theta(\boldsymbol{\phi})$, we use the inner convex approximation method~\cite{scutari2016parallel} to solve the relaxed JCP control problem by optimizing a sequence of strongly convex inner approximations of $\Theta(\boldsymbol{\phi})$ in the form: given $\boldsymbol{\phi}^\kappa \in \Phi$
\begin{align}\label{appro_func}
    \Theta(\boldsymbol{\phi}, \boldsymbol{\phi}^\kappa) = \Theta_1(\boldsymbol{\phi})\Theta_2(\boldsymbol{\phi}^\kappa) + \Theta_1(\boldsymbol{\phi}^\kappa)\Theta_2(\boldsymbol{\phi}),
\end{align}
where $\boldsymbol{\phi}^\kappa  = \{H^\kappa, {p_b}^\kappa \}$ refers to the intermediate $\boldsymbol{\phi}$ obtained in the $\kappa$-th iteration. Obviously, the approximated objective function in (\ref{appro_func}) is strongly convex with the fixed $\boldsymbol{\phi}^\kappa$. With the surrogate function above, we are essentially required to compute the optimal solutions of the following convex optimization problem in each iteration, while preserving the feasibility of the iterates to the original problem in (\ref{OrigProb}).
\begin{subequations}  \label{ApproProb}
\begin{align}
\min_{\substack{p_b,H}}  \quad & \Theta(\boldsymbol{\phi}, \boldsymbol{\phi}^\kappa)  \\
s.t.\quad & \quad 0 < p_b \leq p^{max}_b, \\
    \quad & H_{min} \leq H \leq H_{max}.  
\end{align}
\end{subequations}

Notice that the problem (\ref{ApproProb}) can be solved by various commercial solvers, e.g., IBM CPLEX optimizer~\cite{IBMopt}. The formal description of the Joint Power and Aggregation Control Algorithm is presented in Alg.~\ref{JCP}. Starting from a feasible point $\boldsymbol{\phi}^0$, the method consists in iteratively computing the solution $\boldsymbol{\phi}^*(\boldsymbol{\phi}^\kappa)$ to the surrogate problem (\ref{ApproProb}), and then taking a step from $\boldsymbol{\phi}^\kappa$ towards $\boldsymbol{\phi}^*(\boldsymbol{\phi}^\kappa)$. The process is repeated until it meets the termination criterion, and the value of $H$ is rounded afterward to ensure its feasibility.

\section{Performance Evaluation}
%\begin{figure} \centering %\hspace{-3em}
%  \subfigure[Constellation.\label{fig:cons}]
%  {\includegraphics[width=3.8cm]{fig/constellation.png}}
%  \subfigure[Received symbols.\label{fig:error}]
%  {\includegraphics[width=4.8cm]{fig/error.png}}
%  \caption{M-AirComp demo results.} \label{Fig:M-Air}
%\end{figure}

\subsection{Implementation of M-AirComp}
As shown in Fig.~\ref{fig:testbed}, we first set up experiments to elaborate on the usage of M-AirComp for a FL testbed. The system comprises one edge server and two edge devices. We let one RTX-8000 server with one USRP X310 play the role of the over-the-air FL aggregator. Each FL client consists of the NVIDIA Jetson TX2 as the computation unit and USRP N210 as the wireless transmitter. We also use WBX 50-2200 MHz Rx/Tx USRP daughterboards, with up to 200 mW output power. The synchronization is provided by USRP X310 REF and PPS output ports through cable connection. In the end, all USRPs are connected to an internet switch. We run MATLAB codes from the Communication Toolbox Support Package for USRP Radio to control the transmitting and receiving in different sessions on the RTX-8000 server.

%\begin{figure}
%    \centering
%    \includegraphics[width = 0.7\linewidth]{fig/testbed.png}
%    \caption{Federated learning via M-AirComp testbed in the lab.}
%   \label{fig:testbed}
%\end{figure}

\begin{figure}[t]
\centering
\begin{minipage}[t]{0.48\textwidth}
\centering
{\includegraphics[width=0.95\linewidth]{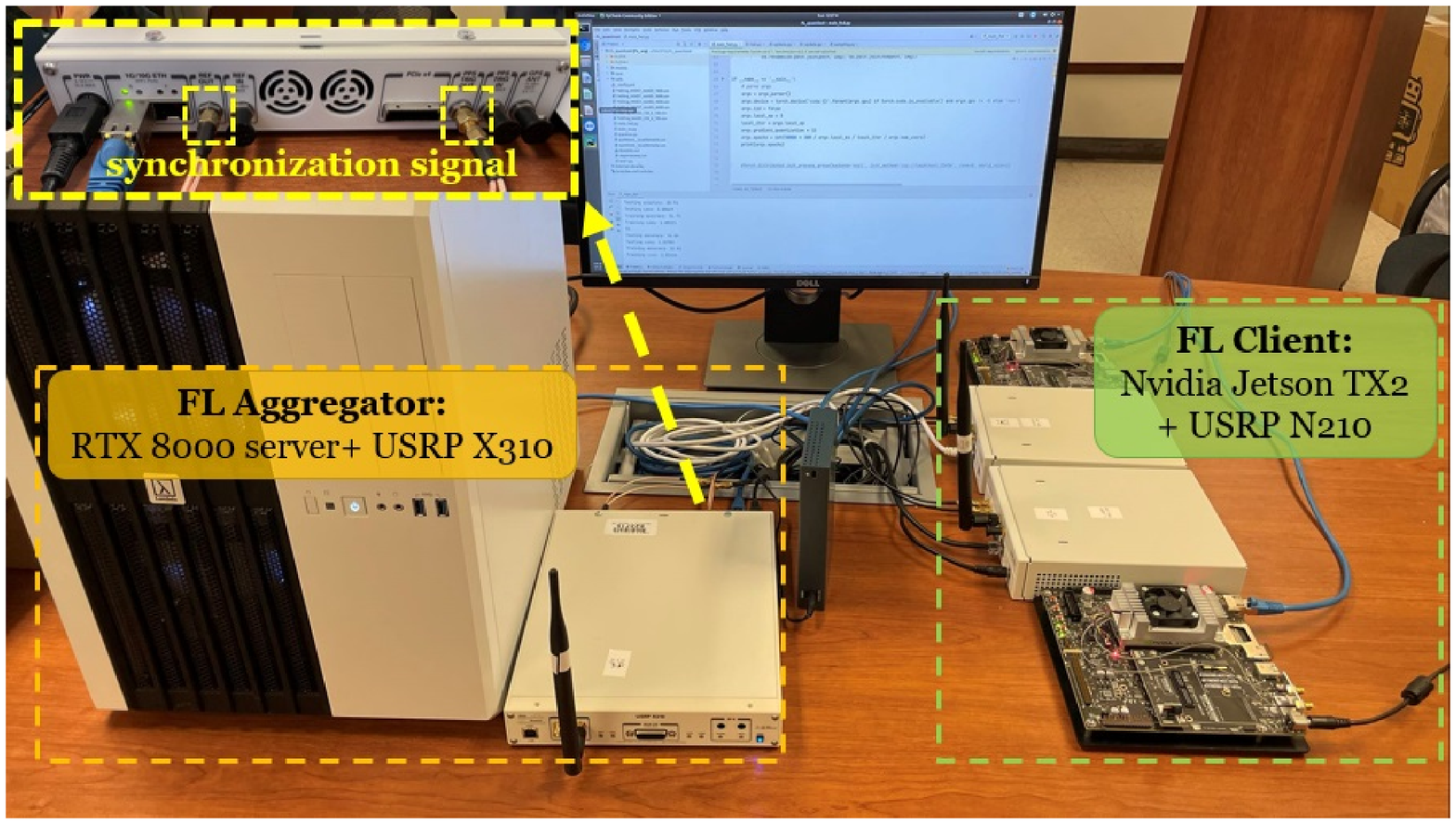}}
 \caption{The testbed of FL with M-AirComp.} \label{fig:testbed}
\end{minipage}
\begin{minipage}[t]{0.48\textwidth}
\centering
\includegraphics[width=0.99\linewidth]{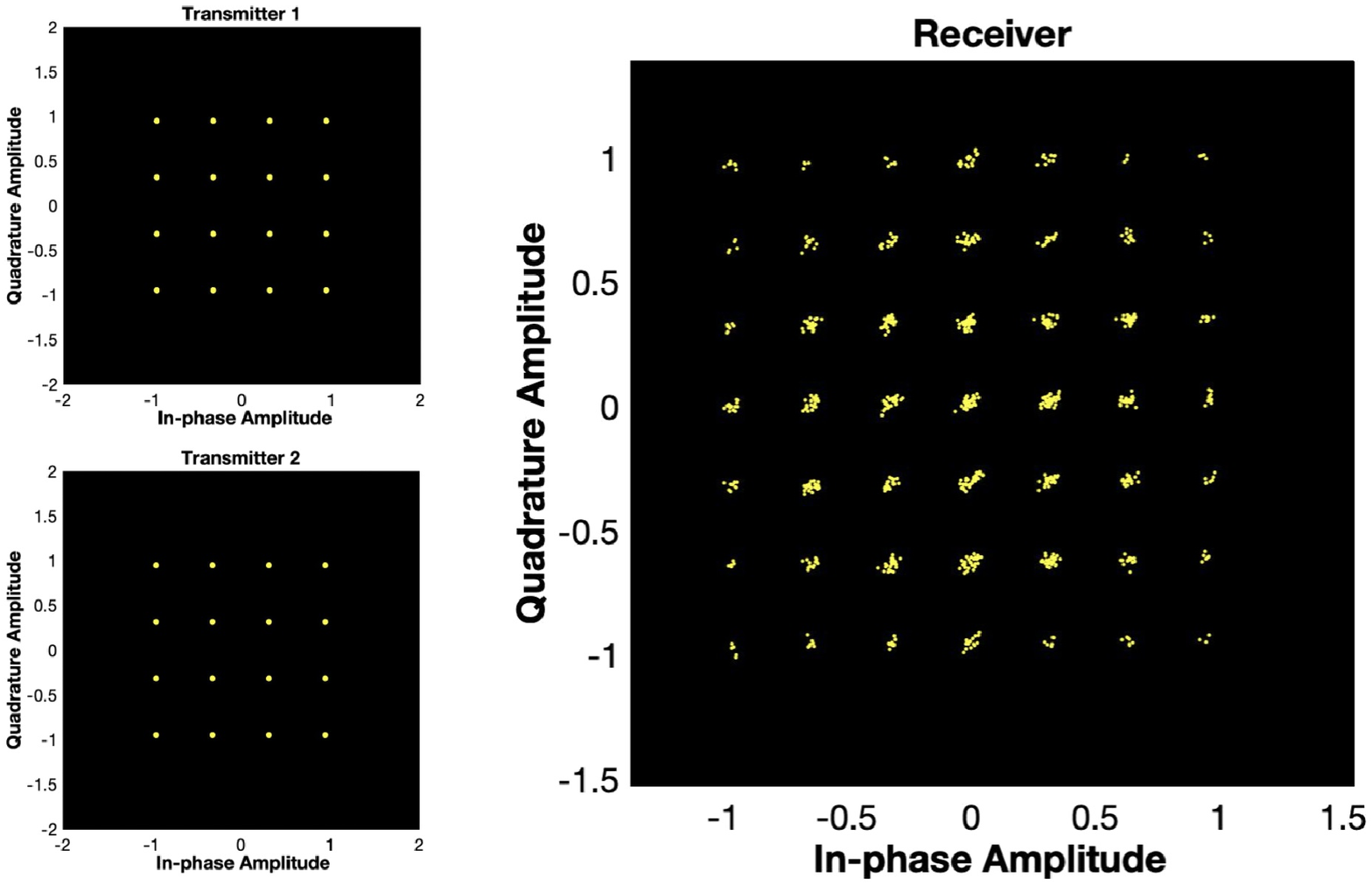}
\caption{Constellation diagram of M-AirComp demo (left: transmitter; right: receiver).} \label{Fig:M-Air}
\end{minipage}
\end{figure}
We first verify the feasibility of M-AirComp by the in-lab experiments. In our M-AirComp demo incorporated with quantization, two edge devices are transmitting QAM symbols, for example, 16 QAM for 4-bit quantization. From the constellation in Fig.~\ref{Fig:M-Air}, the receiving symbol set is expanded into a constellation for higher-order modulations, which explains the addition carried in the over-the-air computation from the communication point of view. The aggregated symbol will be further decoded as a quantized model update, with a certain probability of bit error with regards to the signal-to-noise ratio (SNR).

%\begin{figure} [htbp]\centering
  %{\includegraphics[width=\linewidth]{System.eps}}
%  {\includegraphics[width=0.7\linewidth]{fig/ota_result.png}}
%\caption{Constellation diagram of M-AirComp demo (left: transmitter; right: receiver).} \label{Fig:M-Air}
%\end{figure}

\subsection{Some Observations of the ESOAFL}
As we have discussed in Sec.~\ref{motiv_aircompfl}, AirComp can dramatically improve the spectrum efficiency in the FL training process. In addition, if the communication environment (i.e., channel condition) is extremely poor, our proposed ESOAFL approach can still retain the performance in the case of many participating devices. We consider a severe communication environment with a SNR = 5dB over different numbers of participants(e.g., K = 10, 20, and 30). Here, we train the ResNet-20 model with the CIFAR-10 dataset. As shown in Fig.~\ref{fig:poor_condition}, with the increasing number of participating devices, the convergence gap between the ESOAFL approach and its ideal case (i.e., FedAvg without channel noise) gradually decreases. This verifies that the AirComp variance is decreasing with the number of participating devices $K$, which is also shown in Eq.~\ref{AirComp_variance}. Moreover, especially with a large number of participants (e.g., K=30), the training curve of the ESOAFL approach is similar to its ideal case (i.e., FedAvg) in terms of training epochs, which also exhibits the strong anti-interference ability of the proposed ESOAFL approach.  

\begin{figure}[t]
\centering
\begin{minipage}[t]{0.48\textwidth}
\centering
{\includegraphics[width=0.99\linewidth]{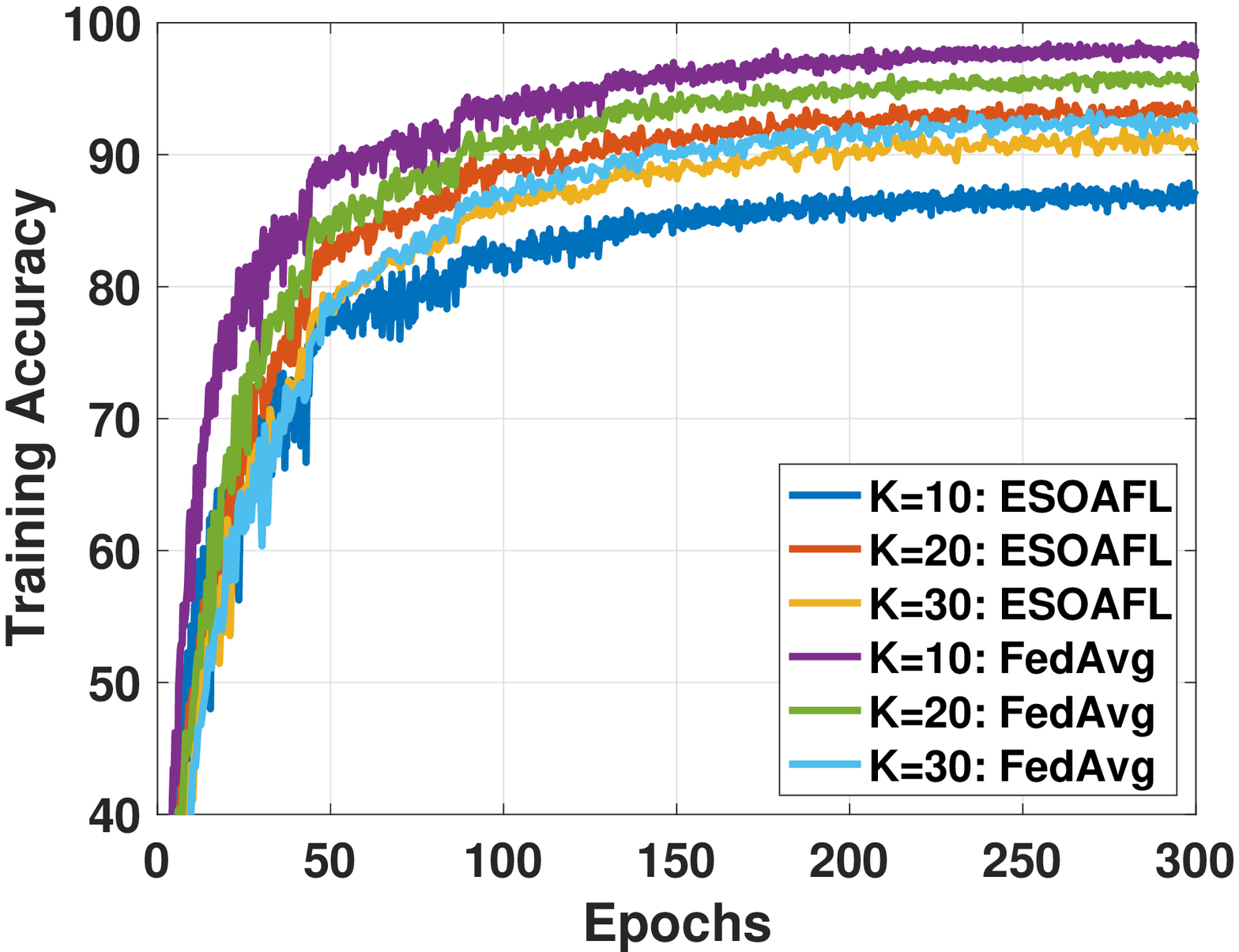}}
 \caption{Training performance under poor channel conditions.} \label{fig:poor_condition}
\end{minipage}
\begin{minipage}[t]{0.48\textwidth}
\centering
\includegraphics[width=0.99\linewidth]{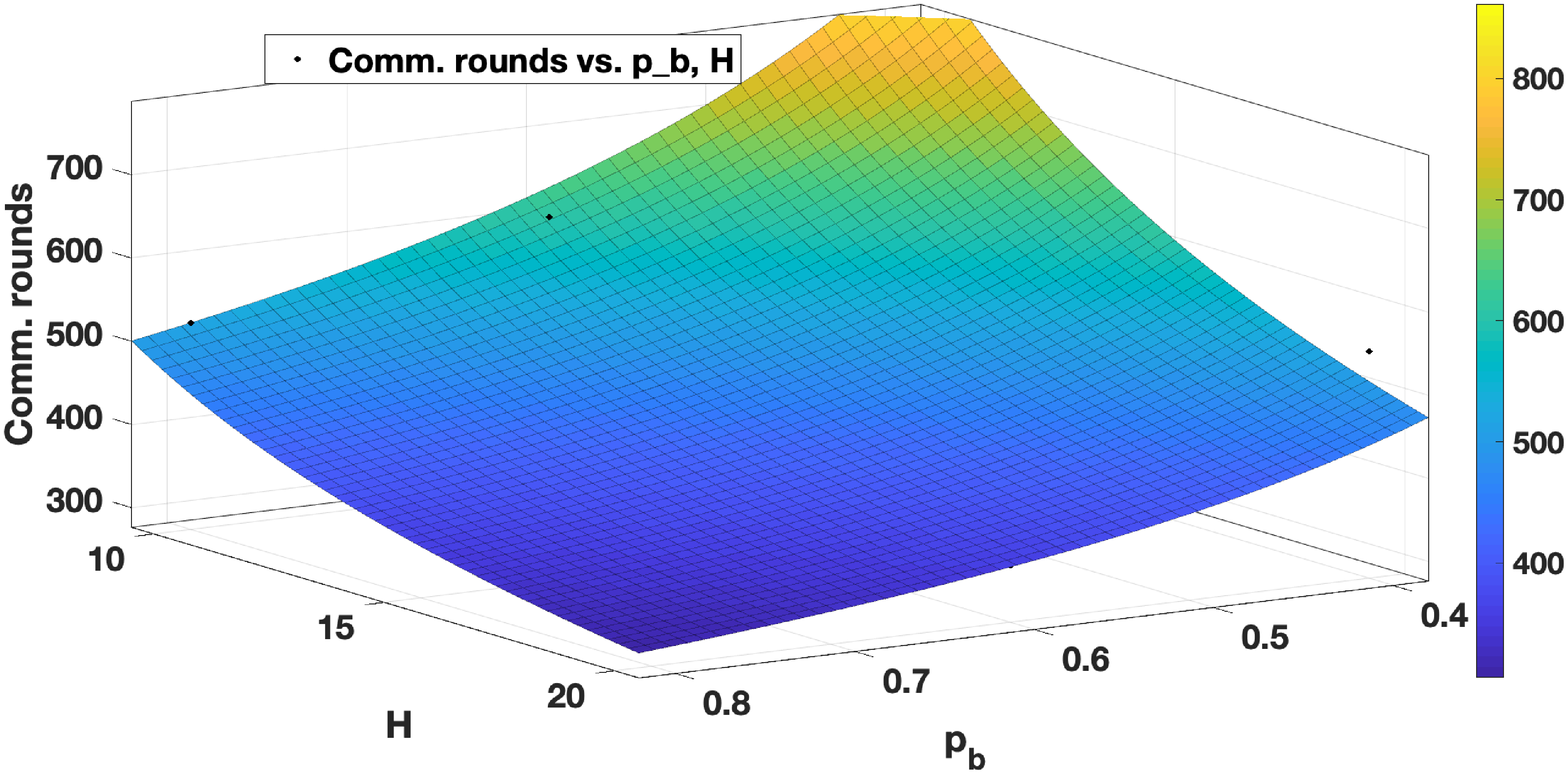}
\caption{Comm. rounds vs. $p_b$ and H.} \label{fig:fitting}
\end{minipage}
\end{figure}

%One difficulty in the problem-solving process is to estimate the values of $A_0$, $B_0$, and $C_0$, which are related to the specific learning model and dataset.Note that the estimation overhead is marginal. 
Recall that the parameters $A_0$, $B_0$, and $C_0$ exist in the JCP control problem, which are related to the specific learning model and dataset. Here, we use a sampling-based method to estimate these parameters, where we empirically sample different combinations ($H$, $p_b$) and employ the derived bound in (\ref{big_O}) to infer their values. Take the ResNet-20 model with the CIFAR-10 dataset as an example. We first implement various local computing iterations $H$ and different transmission probabilities $p_b$ for the training task. Then we set the target training loss and record the corresponding number of communication rounds. After receiving these records, we utilize the Non-linear least squares curve fitting algorithm~\cite{lawson1995solving} to estimate the values of $A_0$, $B_0$, and $C_0$, which are shown in Fig.~\ref{fig:fitting}. We observe that, with the increase of local computing iterations $H$ and transmission probability $p_b$, the number of required communication rounds required is decreasing, but this effect is gradually weakened. At the same time, the computing energy consumption of each round increases linearly with the incremental of local computing iterations $H$. Thus, the energy trade-off between local computing and wireless communications has to be considered to minimize the overall energy consumption, where $H$ and transmission probability $p_b$ are required to be carefully selected.

\subsection{Spectrum and Energy Efficiency of the ESOAFL}

After the parameter estimation, we implement the proposed JCP control scheme to find the optimal local computing iterations $H$ and transmission probability $p_b$. Here, we consider two different image classification models and datasets to verify the effectiveness of our proposed approach, where the LeNet model on the MNIST dataset is relatively light, and ResNet-20 on CIFAR-10 is relatively complex. Both datasets consist of 50000 training images and 10000 test images in 10 classes, and we set batch size as $128$ and $32$ for ResNet and LeNet, respectively. In each round of FL, we set $K = 10$ participating mobile devices executing $H$ steps of SGD in parallel, and the maximum transmission probability $p_b^{\max}$ is set to $0.77$ according to the simulated communication environment and the power constraint. The initial learning rate is $\eta = 0.2$ with a fixed decay rate. We consider several popular FL schemes as baseline approaches compared with our proposed ESOAFL-OPT approach (i.e., ESOAFL with optimal JCP control).
\begin{itemize}
    \item FedAvg~\cite{mcmahan2017communication}: FL without AirComp, where the ideal noise-free transmission is supposed.
    \item FedPAQ~\cite{reisizadeh2020fedpaq}: the participants transmit the quantized model updates.
    \item OBDA-ADV~\cite{Zhu2020twconebit}: a modified version of the OBDA (one-bit digital AirComp), where we improve the original scheme by ignoring the quantization at the receiver to preserve the learning precision.
    \item ESOAFL-MAX: the proposed ESOAFL scheme without the transmission control, where we adopt the maximum transmission probability $p_b^{\max}$ to transmit the model updates.
\end{itemize}

%\begin{figure} \centering
  %{\includegraphics[width=\linewidth]{System.eps}}
%  {\includegraphics[width=0.8\linewidth]{fig/curve_fitting.eps}}
% \caption{Comm. rounds vs. $p_b$ and H.} \label{fig:fitting}
%\end{figure}

We assume the same communication bandwidth for all the schemes. We utilize the Nvidia TX2 as the mobile device and deploy Jtop~\cite{jtop} tool to measure the computing energy, where the LeNet model consumes $0.03$J, and the ResNet model consumes $0.5$J for one training iteration. For example, training the ResNet model for one iteration consumes $130$ms, and the GPU power is nearly $4$W. We assume the AirComp can be deployed in the commercial LTE system for wireless transmissions. The resource block is $180$ kHz, and we can obtain the transmission time of local updates with the specific model size accordingly. Moreover, we assume the average maximum transmit power is $0.2$W. Thus, the transmission energy consumption can be calculated as the product of transmit power and transmission time. In all schemes, We set the average SNR=15dB for participants, whose channel quality can be reflected by the CQI (Channel Quality Indicator) category 11. In this case, the modulation scheme, code rate, bits per resource element are 64QAM, 0.8525, 5.115, respectively, in FedAvg and FedPAQ for a fair comparison.

%We assume all schemes can utilize the same amount of communication bandwidth. FedAvg represents the FL approach without AirComp, where we assume the ideal transmission without channel noise is taken for FedAvg. OBDA-ADV is a modified version of the OBDA (one-bit digital AirComp), where we improve the original scheme with some advanced techniques to improve communication efficiency. The details of the OBDA-ADV will be illustrated in the following paragraphs. ESOAFL-MAX is the proposed ESOAFL scheme without the transmission control, where we adopt the maximum transmission probability $p_b^{\max}$ to transmit gradients.

\begin{figure*} \centering %\hspace{-3em}
  \subfigure[Training loss vs. comm. resources\label{Fig:exp1}]
  {\includegraphics[width=.46\textwidth]{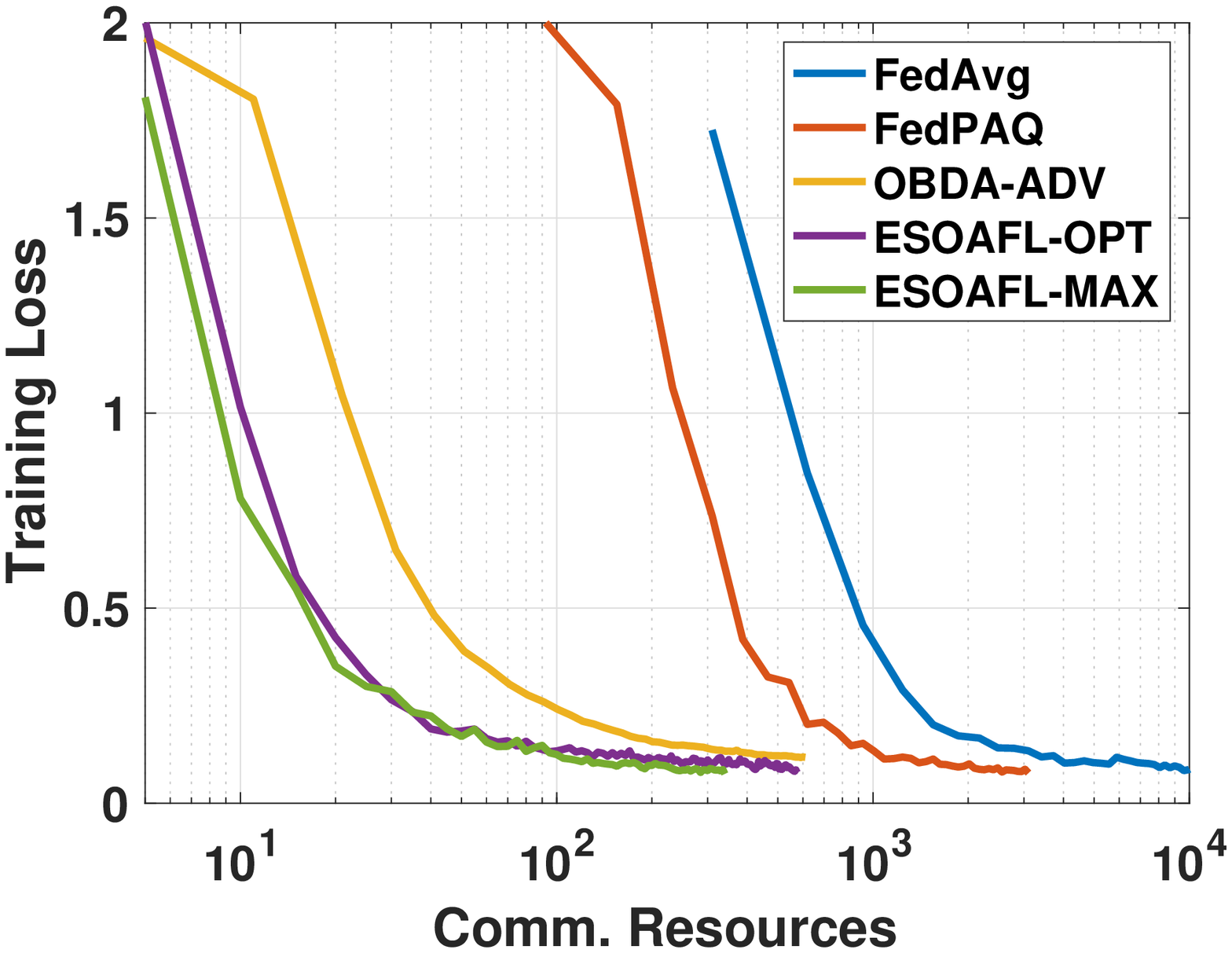}}
  \subfigure[Training acc. vs. energy cons.\label{Fig:exp2}]
  {\includegraphics[width=.46\textwidth]{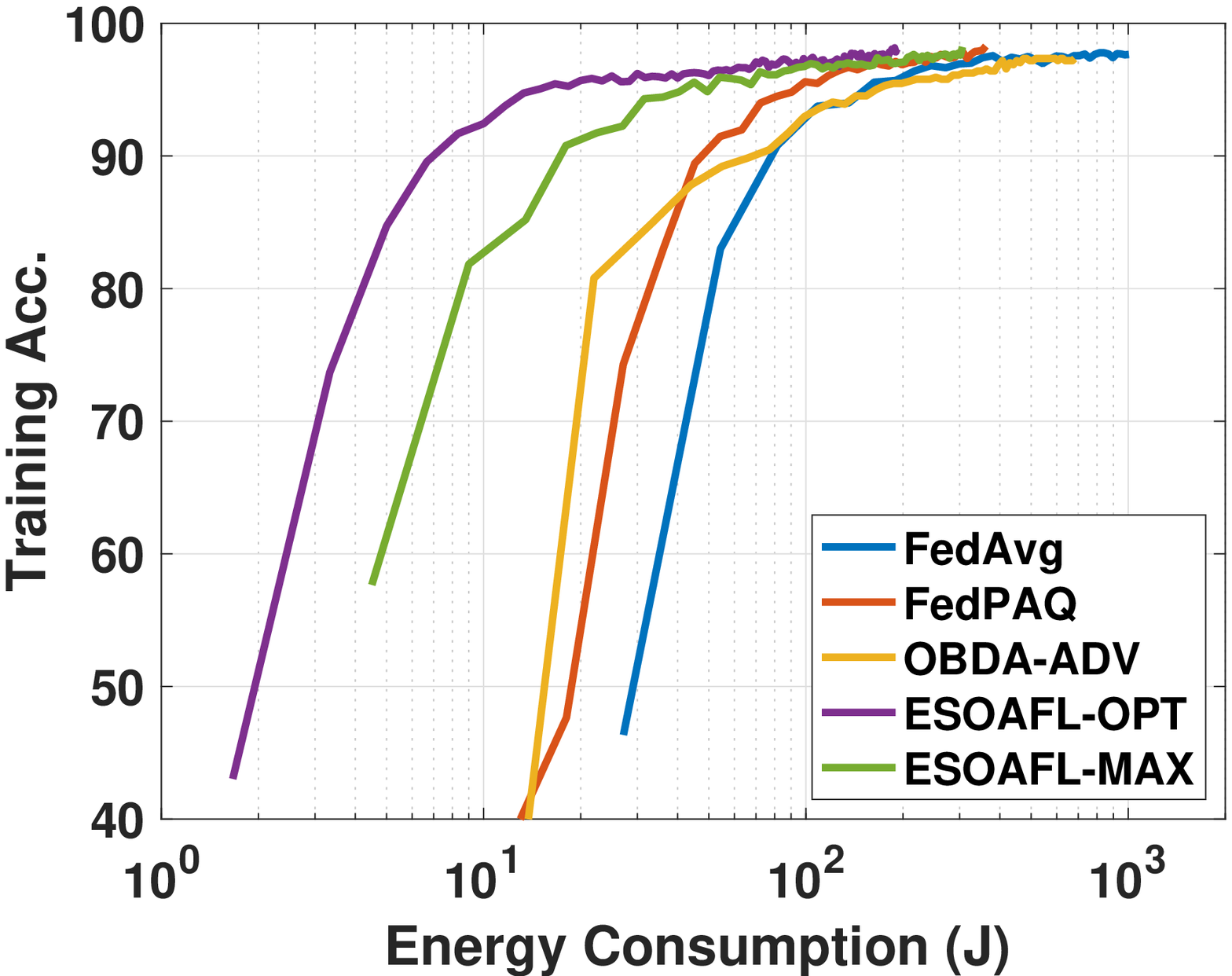}}
  \caption{Training performance of LeNet on MNIST dataset. } \label{Fig:exp11}
\end{figure*}

Fig.~\ref{Fig:exp1} and \ref{Fig:exp2} show the simulation results for LeNet on MNIST. Here, we set the target training loss $\epsilon$ as $0.07$ and assume the data samples are independent and identically distributed (IID). For the OBDA-ADV scheme, we bring the local SGD method (i.e., taking several training steps among the sequential communication rounds) into the original scheme. Let the spectrum resource consumed in each round of ESOAFL be the unit communication resource. We set the gradient quantization level as $4$-bit in ESOAFL and FedPAQ. Fig.~\ref{Fig:exp1} illustrates the communication resources consumption during the training procedure, and we can obviously find that the proposed ESOAFL significantly improves the spectrum efficiency compared with FedAvg, FedPAQ. One reason is that FedAvg and FedPAQ consume more communication resources in each round because all devices cannot take the concurrent transmission with the same bandwidth. Another reason is that each pair of gradients can be transmitted orthogonally using in-phase (I) and quadrature (Q) channels simultaneously in the proposed ESOAFL scheme. However, according to the LTE protocol, each resource element can only carry several bits of a gradient in FedAvg and FedPAQ. In the meantime, Fig.~\ref{Fig:exp2} presents the energy consumption during FL training, where $H=3$ and $p_b=0.29$ is obtained for optimal controlling of ESOAFL. The results show that our ESOAFL scheme consumes the least energy among all schemes. Specifically, when achieving the same target training loss, the energy efficiency of ESOAFL-OPT is twice and three times higher than that of FedPAQ and OBDA-ADV, respectively. This is because the energy efficient power control policy and the digital modulation scheme in the M-AirComp design save both the transmit power and time. Moreover, since the optimized transmission probability is much lower than the maximum value, our ESOAFL-OPT approach only consumes nearly half of the ESOAFL-MAX approach's energy, which demonstrates the necessity of the JCP control scheme. Note that the low-precision OBDA-ADV approach cannot reach the target training loss we set, and thus we consider the training loss $\epsilon=0.12$ for the OBDA-ADV approach. 

%Thus, the FedAvg approach with $10$ devices consumes more communication resources in each round since all devices cannot take the concurrent transmission with the same bandwidth. This is because that the low precision of the OBDA-ADV hinders the training convergence, can save nearly $4\times$ ($6\times$) energy consumption than the FedPAQ (OBDA-ADV) to achieve the same target training loss. 

\begin{figure*} \centering %\hspace{-3em}
  \subfigure[Training loss vs. comm. resources\label{Fig:exp3}]
  {\includegraphics[width=.46\textwidth]{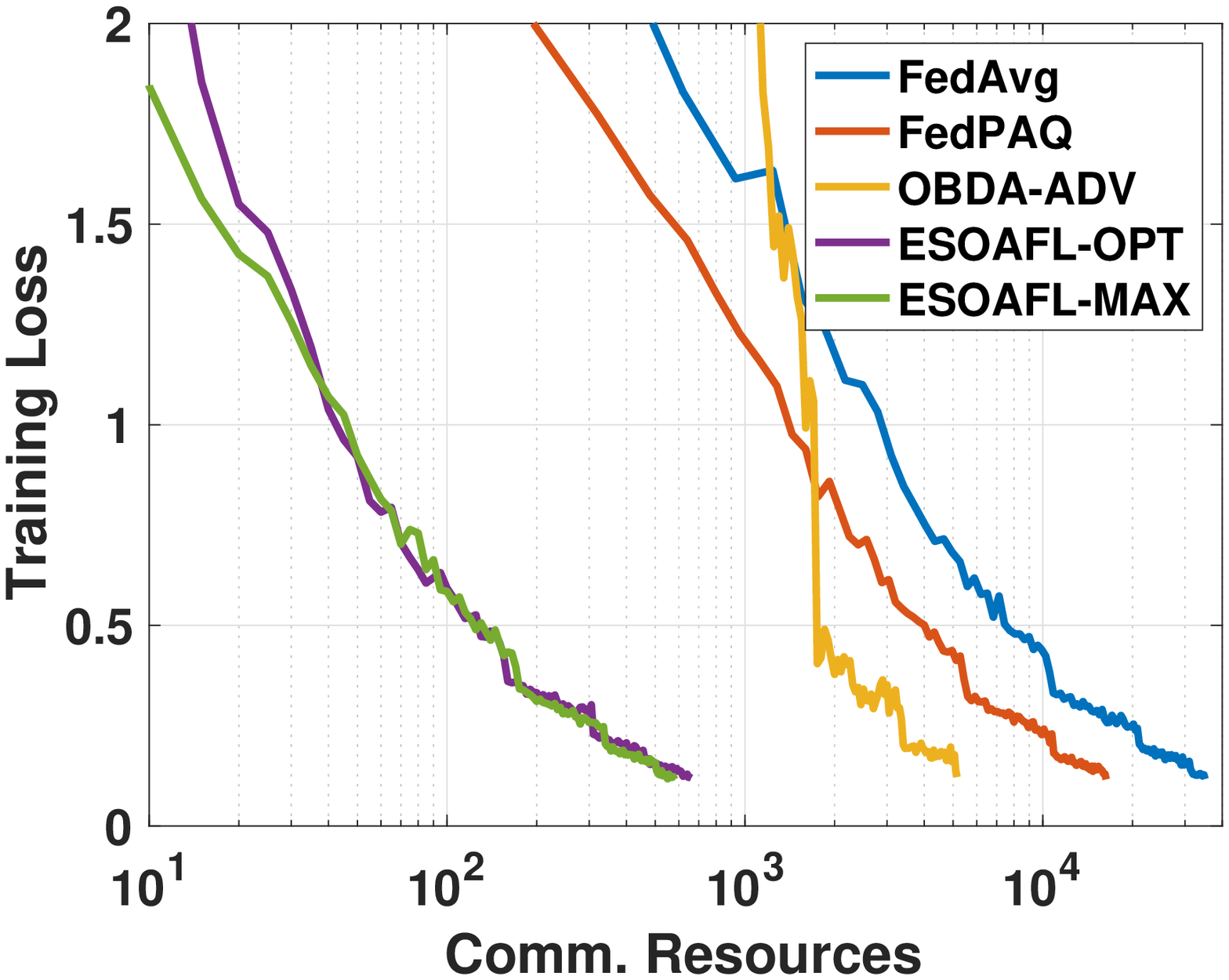}}
  \subfigure[Training acc. vs. energy cons.\label{Fig:exp4}]
  {\includegraphics[width=.46\textwidth]{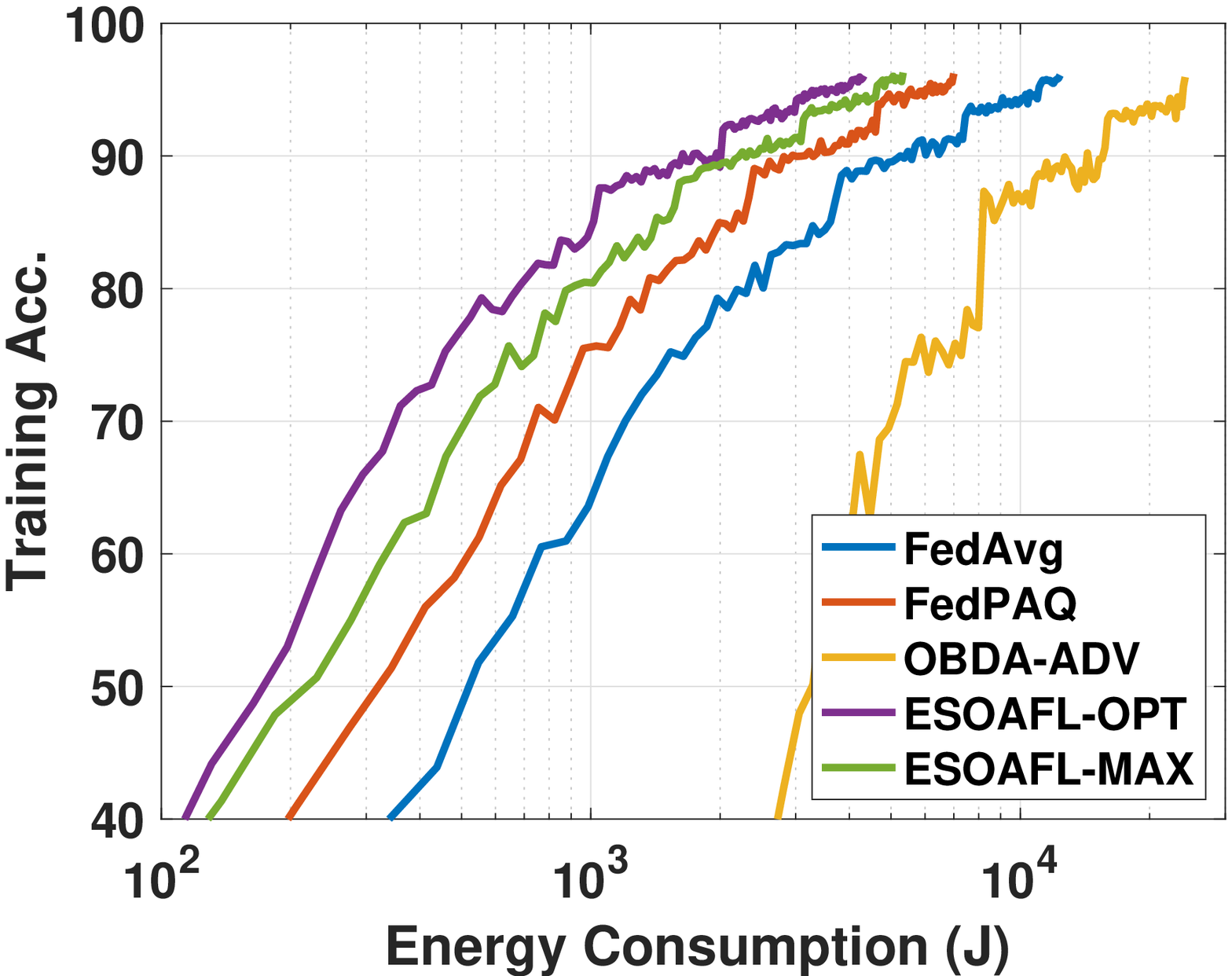}}
  \caption{Training performance of ResNet-20 on CIFAR-10 dataset.} \label{Fig:exp22}
\end{figure*}

\begin{table}
\centering
\caption{Performance comparison under different learning settings (ResNet20 on CIFAR-10)}
\label{tab:final}
\begin{tabular}{|c|c|c;{1pt/1pt}c;{1pt/1pt}c;{1pt/1pt}c;{1pt/1pt}c;{1pt/1pt}c|} 
\hline
\multicolumn{2}{|c|}{\multirow{2}{*}{}}                                                                                                                                                                                         & \multicolumn{3}{c|}{K=10, B=128, H=10}                                                                                                        & \multicolumn{3}{c|}{K=10, B=128, H=10}                                                                                                                                                               \\ 
\cline{3-8}
\multicolumn{2}{|c|}{}                                                                                                                                                                                                          & \multicolumn{1}{c|}{\textbf{\textbf{FedAvg}}} & \multicolumn{1}{c|}{\textbf{\textbf{FedPAQ}}} & \multicolumn{1}{c|}{\textbf{\textbf{ESOAFL}}} & \multicolumn{1}{c|}{\textbf{\textbf{\textbf{\textbf{FedAvg}}}}} & \multicolumn{1}{c|}{\textbf{\textbf{\textbf{\textbf{FedPAQ}}}}} & \multicolumn{1}{c:}{\textbf{\textbf{\textbf{\textbf{ESOAFL}}}}}  \\ 
\hline
\multirow{3}{*}{\textbf{\textbf{IID}}}                                                                                                                        & \textbf{\textbf{Comm.}}                                         & 35030                                         & 16320                                         & \textbf{\textbf{656}}                         & 350300~                                                         & 187200                                                          & \textbf{\textbf{600}}                                            \\ 
\cline{2-2}\cdashline{3-8}[1pt/1pt]
                                                                                                                                                              & \textbf{\textbf{Energy}}                                        & 12379                                         & 7011                                          & \textbf{\textbf{\textbf{\textbf{4323}}}}      & 9977                                                            & 5544                                                            & \textbf{\textbf{1530}}                                           \\ 
\cline{2-2}\cdashline{3-8}[1pt/1pt]
                                                                                                                                                              & \textbf{\textbf{Acc.}}                                          & $88.1\%$                                      & $88.1\%$                                      & $87.4\%$                                      & $87.7\%$                                                        & $87.3\%$                                                        & $87.1\%$                                                         \\ 
\hline
\multirow{3}{*}{\begin{tabular}[c]{@{}c@{}}\textbf{\textbf{Non-IID:~}}\\\textbf{\textbf{$\varsigma = 0.3$}}\end{tabular}}                                     & \textbf{\textbf{\textbf{\textbf{Comm.}}}}                       & 37820                                         & \multicolumn{1}{c:}{17632}                    & \textbf{\textbf{674}}                         & 418500                                                          & 214400                                                          & \textbf{\textbf{675}}                                            \\ 
\cline{2-2}\cdashline{3-8}[1pt/1pt]
                                                                                                                                                              & \textbf{\textbf{\textbf{\textbf{Energy}}}}                      & 13365                                         & 7550                                          & \textbf{\textbf{4590}}                        & 11920                                                           & 6349                                                            & \textbf{\textbf{1721}}                                           \\ 
\cline{2-2}\cdashline{3-8}[1pt/1pt]
                                                                                                                                                              & \textbf{\textbf{\textbf{\textbf{Acc.}}}}                        & $87.9\%$                                      & $87.4\%$                                      & $87.1\%$                                      & $86.7\%$                                                        & $86.6\%$                                                        & $86.4\%$                                                         \\ 
\hline
\multirow{3}{*}{\begin{tabular}[c]{@{}c@{}}\textbf{\textbf{\textbf{\textbf{Non-IID:~}}}}\\\textbf{\textbf{\textbf{\textbf{$\varsigma = 0.3$}}}}\end{tabular}} & \multicolumn{1}{l|}{\textbf{\textbf{\textbf{\textbf{Comm.}}}}}  & 42470                                         & 22080                                         & \textbf{\textbf{693}}                         & 461900                                                          & 238400                                                          & \textbf{\textbf{740}}                                            \\ 
\cline{2-2}\cdashline{3-8}[1pt/1pt]
                                                                                                                                                              & \multicolumn{1}{l|}{\textbf{\textbf{\textbf{\textbf{Energy}}}}} & 15008                                         & 9471                                          & \textbf{\textbf{4692}}                        & 13156                                                           & 7060                                                            & \textbf{\textbf{1887}}                                           \\ 
\cline{2-2}\cdashline{3-8}[1pt/1pt]
                                                                                                                                                              & \multicolumn{1}{l|}{\textbf{\textbf{\textbf{\textbf{Acc.}}}}}   & $85.1\%$                                      & $85.1\%$                                      & $84.8\%$                                      & $81.0\%$                                                        & $81.0\%$                                                        & $81.0\%$                                                         \\ 
\hline
\multirow{3}{*}{\begin{tabular}[c]{@{}c@{}}\textbf{\textbf{\textbf{\textbf{Non-IID:~}}}}\\\textbf{\textbf{\textbf{\textbf{$\varsigma = 0.3$}}}}\end{tabular}} & \multicolumn{1}{l|}{\textbf{\textbf{\textbf{\textbf{Comm.}}}}}  & 53010                                         & 27040                                         & \textbf{\textbf{861}}                         & 492900                                                          & 254400                                                          & \textbf{\textbf{785}}                                            \\ 
\cline{2-2}\cdashline{3-8}[1pt/1pt]
                                                                                                                                                              & \multicolumn{1}{l|}{\textbf{\textbf{\textbf{\textbf{Energy}}}}} & 18951                                         & 11600                                         & \textbf{\textbf{5848}}                        & 14039                                                           & 7535                                                            & \textbf{\textbf{785}}                                            \\ 
\cline{2-2}\cdashline{3-8}[1pt/1pt]
                                                                                                                                                              & \multicolumn{1}{l|}{\textbf{\textbf{\textbf{\textbf{Acc.}}}}}   & $68.4\%$                                      & $68.6\%$                                      & $68.1\%$                                      & $63.1\%$                                                        & $58.9\%$                                                        & $53.2\%$                                                         \\
\hline
\end{tabular}
\end{table}

Fig.~\ref{Fig:exp3} and \ref{Fig:exp4} demonstrate the performance comparison of all schemes with ResNet-20 model on CIFAR-10 dataset. We set the target training loss $\epsilon$ as $0.12$, and obtain the optimal control strategies $H = 11$ and $p_b = 0.51$. Like the conclusions described above, the proposed ESOAFL approach dramatically improves the spectrum efficiency and reduces the required energy. In this situation, our proposed ESOAFL-OPT saves hundreds of times of communication resources compared with FedAvg and FedPAQ. It also saves more than $8\times$ of communication resources compared with the OBDA method. Accordingly, our proposed ESOAFL-OPT scheme saves nearly one-third and two-thirds of energy consumption than FedPAQ and FedAvg schemes. Furthermore, the OBDA-ADV approach has relatively poor convergence performance compared with other approaches due to the high precision requirement of the complex ResNet-20 model. 

%improves the energy efficiency by $4$ and $8$ times compared with FedPAQ and FedAvg

We further show the scalability of the ESOAFL scheme with more learning settings. Here, we consider different data distributions in the content of different levels of non-IID data. Let $\varsigma \in [0,1]$ denotes the non-IID level~\cite{wang2020infocom}. For example, $\varsigma = 0.3$ indicates that $30\%$ of the data belong to one label and the remaining $70\%$ data belong to others. We ignore the OBDA-ADV scheme since its performance is not good in non-IID data settings. From Table.~\ref{tab:final}, we can observe that training with non-IID data incurs a larger energy and communication resources consumption to converge. Moreover, compared with FedAvg and FedPAQ, the proposed ESOAFL achieves the indistinguishable final testing accuracy at all non-IID levels while saving communication resources and overall energy consumption. We also conducted the simulations with $K=100$ participants, obtaining similar observations. Note that compared with $K=10$ participants settings, we put less computing loads ($B=32$, $H=5$) in each communication round of the $K=100$ setting, thus causing more communication loads. Therefore, the communication resources consumption of FedAvg and FedPAQ at $K=100$ increases significantly compared with the scenario $K=10$. However, because of the concurrent property of the AirComp, communication resources consumption at $K=100$ in the ESOAFL scheme remains constant compared with the scenario of $K=10$, which indicates the potential of involving large mounts of participants in the ESOAFL scheme.

%In this situation, OBDA-ADV refers to the OBDA scheme without quantizing received gradients at the receiver, thus preserving the learning precision.
% \usepackage{multirow}
% \usepackage{booktabs}

\section{Related Works}

Recently, much attention has been paid to the energy-efficient FL over mobile devices, where several advanced techniques are utilized to save energy during the FL training~\cite{shi2021towards}. On the one hand, gradient sparsification~\cite{stich2018sparsified, alistarh2018convergence} and gradient quantization~\cite{alistarh2017qsgd,tang2018communication} techniques can compress model updates in the transmission process, significantly reducing the communication burdens~\cite{Li2021infocom}. On the other hand, some researchers consider applying a weight quantization scheme to reduce the required computing energy~\cite{fu2020don}. Although these methods can effectively reduce the energy cost, they are mainly considered from the perspective of learning algorithms and widely ignore the communication components, especially with the physical layer aspects of communication. Realizing the above problem, some pioneering works exploit the waveform superposition property of the wireless medium and propose the AirComp FL~\cite{Zhu2020twcbaa}. Cao et al. in~\cite{cao2020twcopcota}, and Amiri and Gündüz in~\cite{Gunduz2020twcfloefc} apply AirComp to solve the communication bottleneck when a large number of participants aggregate the data together, where power allocation schemes are derived to satisfy the mean square error (MSE) requirements. Additionally, the works in~\cite{yang2020twcflota,fan2021joint} propose joint device selection and communication scheme design methods to improve the learning performance for AirComp FL. All these works take analogy modulation schemes for wireless transmission, which are difficult to be implemented on commercial devices. In addition, the convergence analysis for the whole FL training procedure is not discussed in these works. Noticing the limits above, Zhu et al.~\cite{Zhu2020twconebit} applies the 1-bit digital modulation and derives the convergence analysis accordingly. However, 1-bit based scheme tremendously scarifies precision without considering the energy consumption during the training process. Different from the existing approaches, our design targets the general digital modulation scheme with multiple bits, where the convergence-guaranteed FL approach integrates both the AirComp and the gradient quantization techniques. Moreover, the energy consumption issue, including both computing and communication, is well studied accordingly.   

\section{Conclusion}
In this paper, we proposed the ESOAFL scheme for energy and spectrum efficient FL over mobile devices, where M-AirComp was applied for model updates transmission in a joint compute-and-communicate manner. A high-precision digital modulation scheme with multi-bit gradient quantization was designed for the participating devices to upload their model updates during FL. With the theoretical convergence analysis of the modified FL algorithm, we further developed a joint local computing and transmission probability control approach aiming to minimize the overall energy consumed by all devices. Extensive simulations were conducted to verify our theoretical analysis, and the results showed that the ESOAFL scheme effectively improves the spectrum efficiency with the learning precision guaranteed. Besides, it also saved at least half of energy consumption compared with other FL methods.

%\section*{Acknowledgments}
%This should be a simple paragraph before the References to thank those individuals and institutions who have supported your work on this article.

\bibliographystyle{IEEEtran}
\bibliography{ref_dian.bib}

\section*{Appendix}
\subsection{Proof of Theorem 1}\label{proof_theorem1}
 
Consider a non-convex FL model setting. Under the L-smoothness assumption of the global objective with the expectation taken, we have 

\begin{align}\label{conv_bound_1}
\begin{split}
    &\mathbb{E}\left[\mathbb{E}_{Q}\left[\mathbb{E}_{\text{Air}}\left[f(\mathbf{w}^{r+1})-f(\mathbf{w}^{r}) \right]\right]\right] \\
\leq& -\theta \eta \mathbb{E}\left[\mathbb{E}_{Q}\left[\mathbb{E}_{\text{Air}}\left[ \left\langle\nabla f^r, \nabla F_{Q}^{r}\right\rangle \right]\right]\right] +\frac{\theta ^2\eta^2L}{2}\mathbb{E}\left[\mathbb{E}_{Q}\left[\mathbb{E}_{\text{Air}}\left[ \|\nabla F^r_Q\|^2 \right]\right]\right],
\end{split}
\end{align}
where we take the expectation over the sampling and operations. Next, the following lemmas are proposed to bound terms in the above inequality.

\begin{lemma}\label{lemma_1}
The inner product between the stochastic gradient $\nabla F_{Q}^{r}$ and full batch gradient $\nabla f^r$ can be bounded as
\begin{flalign}
    \mathbb{E}_{\xi^{(r)}} \mathbb{E}_{Q}\mathbb{E}_{\text{Air}}\left[\left\langle\nabla f^r, \nabla F_{Q}^{r}\right\rangle\right]
    &=\mathbb{E}_{\xi^{(r)}}\left[\left\langle \nabla f^r, \frac{1}{K}\sum_{k=1}^{K}\sum_{h=0}^{H-1}\nabla F_{k}^{r,h}  \right\rangle\right]\nonumber\\
    &\leq \frac{1}{2K}\sum_{k=1}^{K}\sum_{h=0}^{H}\left[-\|\nabla f^r\|^2_2-\|\nabla f^{r,h}_k\|^2_2+L^2\|\mathbf{w}^r-\mathbf{w}^{r,h}_k\|_2^2\right].
\end{flalign}
Here, we set $\nabla F^r_k = \sum_{h=0}^{H-1}\nabla F_k^{r,h}$ and $\nabla F^r_{k,Q} = Q\left(\sum_{h=0}^{H-1} \nabla F_k^{(h, r)}\right)$. We further define $\nabla F^r_Q =\text{Air}_{ \mathcal{K}}\left(Q\left(\sum_{h=0}^{H-1}\nabla F_k^{r,h}\right)\right) $.
\end{lemma}

\begin{lemma}\label{lemma_2}
  Similar to the Lemma D.3 in~[11], we can bound the distance between the global model and the local model at $r$-th communication round under Assumption $2$ as follows:
\begin{align}
    \mathbb{E}\left[\|\mathbf{w}^r-\mathbf{w}^{r,h}_k\|^2_2\right]\leq \eta^2H\sigma^2 + \eta^2\sum_{h=0}^{H-1}H\|\nabla f_k^{r,h}\|^2_2
\end{align}
\end{lemma}

\begin{lemma}\label{lemma_3}
The last term in~(\ref{conv_bound_1}) can be calculated as
\begin{align}
\begin{split}
&\mathbb{E}_{\xi^{(r)}} \mathbb{E}_{Q}\mathbb{E}_{\text{Air}}\left[\|\text{Air}_{ \mathcal{K}}\left(Q\left(\sum_{h=0}^{H-1}\nabla F_k^{r,h}\right)\right)\|^2\right] \leq \frac{\sigma_z^2}{K^2{p_b}^2}\\
& +  \sum_{k=1}^{K}\frac{q+{p_b}}{K^2{p_b}}\operatorname{Var}(\nabla F^r_{k})+ \sum_{k=1}^{K}\frac{q(2-{p_b}) +K{p_b}}{K^2{p_b}}\|\nabla f^r_{k}\|^2
\end{split}
\end{align}
\end{lemma}
\begin{proof}
 
\begin{align}
\begin{split}
&\mathbb{E}_{\xi^{(r)}} \mathbb{E}_{Q}\mathbb{E}_{Air}\left[\|Air_{ \mathcal{K}}\left(Q\left(\sum_{h=0}^{H-1}\nabla F_k^{r,h}\right)\right)\|^2\right]\\
=&\mathbb{E}_{\xi^{(r)},Q}\left[  \frac{1}{K^2}\left(\|\sum_{k=1}^{K}\nabla F^r_{k,Q}\|^2 + (\frac{1}{{p_b}}-1)\sum_{k=1}^{K}\|\nabla F^r_{k,Q}\|^2 \right) + \frac{\sigma_z^2}{K^2{p_b}^2} \right]\\
=&\mathbb{E}_{\xi^{(r)}}\left[\mathbb{E}_{Q}\left[\|\frac{1}{K}\sum_{k=1}^{K}\nabla F^r_{k,Q}\|^2 + \frac{1}{K^2}(\frac{1}{{p_b}}-1)\sum_{k=1}^{K}\|\nabla F^r_{k,Q}\|^2 \right]\right] + \frac{\sigma_z^2}{K^2{p_b}^2} \\
=&\mathbb{E}_{\xi^{(r)}}\left[\mathbb{E}_{Q}\left[\frac{1}{K^2}\sum_{k=1}^{K}\left[\|\nabla F^r_{k,Q} - \nabla F^r_{k}\|^2\right]\right] + \|\frac{1}{K}\sum_{k=1}^{K}\nabla F^r_{k}\|^2 \right]+ \\
&\mathbb{E}_{\xi^{(r)}}\left[\frac{1}{K^2}(\frac{1}{{p_b}}-1)\sum_{k=1}^{K} (\mathbb{E}_{Q}\left[\|\nabla F^r_{k,Q} - \nabla F^r_{k}\|^2\right] + \|\nabla F^r_{k}\|^2\right] + \frac{\sigma_z^2}{K^2{p_b}^2} \\
=&\mathbb{E}_{\xi^{(r)}}\left[\mathbb{E}_{Q}\left[\frac{1}{K^2{p_b}}\sum_{k=1}^{K}\left[\|\nabla F^r_{k,Q} - \nabla F^r_{k}\|^2\right]\right] +   \|\frac{1}{K}\sum_{k=1}^{K}\nabla F^r_{k}\|^2 + \frac{1}{K}(\frac{1}{{p_b}}-1)\sum_{k=1}^{K}\|\nabla F^r_{k}\|^2 \right]+ \frac{\sigma_z^2}{K^2{p_b}^2}\\
\leq & \mathbb{E}_{\xi^{(r)}}\left[\sum_{k=1}^{K}\frac{q}{K^2p}\|\nabla F^r_{k}\|^2 +   \|\frac{1}{K}\sum_{k=1}^{K}\nabla F^r_{k}\|^2 + \frac{1}{K}(\frac{1}{{p_b}}-1)\sum_{k=1}^{K}\| \nabla F^r_{k}\|^2 \right] + \frac{\sigma_z^2}{K^2{p_b}^2}\\
=&\sum_{k=1}^{K}\frac{q}{K^2{p_b}}\left[\operatorname{Var}(\nabla F^r_{k})+ \|\nabla f^r_{k}\|^2\right]+ \left[\frac{1}{K^2}\sum_{k=1}^{K}\operatorname{Var}(\nabla F^r_{k}) + \|\frac{1}{K}\sum_{k=1}^{K}\nabla f^r_{k}\|^2\right]+\\
&\frac{1}{K}(\frac{1}{{p_b}}-1)\sum_{k=1}^{K}\| \nabla f^r_{k}\|^2 + \frac{\sigma_z^2}{K^2{p_b}^2}\\
\leq& \sum_{k=1}^{K}\frac{q+{p_b}}{K^2{p_b}}\operatorname{Var}(\nabla F^r_{k})+ \sum_{k=1}^{K}\frac{q(2-{p_b}) +K{p_b}}{K^2{p_b}}\|\nabla f^r_{k}\|^2+ \frac{\sigma_z^2}{K^2{p_b}^2}
\end{split}
\end{align}
\end{proof}
 According to Assumption 2, we have $\operatorname{Var}(\nabla F^r_{k}) \leq H\sigma^2$. We further have $\| \nabla f^r_{k}\|^2=\|\sum_{h=0}^{H-1}\nabla f^{r,h}_{k}\|^2\leq H\sum_{h=0}^{H-1}\|\nabla f^{r,h}_{k}\|^2$. Therefore, we have
\begin{flalign}
    &\mathbb{E}_{\xi^{(r)}} \mathbb{E}_{Q}\mathbb{E}_{\text{Air}}\left[\|\text{Air}_{ \mathcal{K}}\left(Q\left(\sum_{h=0}^{H-1}\nabla F_k^{r,h}\right)\right)\|^2\right]  \\
    \leq &\frac{q+{p_b}}{K{p_b}}H\sigma^2+H\frac{q(2-{p_b}) +K{p_b}}{K^2{p_b}} \sum_{k=1}^{K} \sum_{h=0}^{H}\|\nabla f^{r,h}_{k}\|^2+ \frac{\sigma_z^2}{K^2{p_b}^2} \nonumber
\end{flalign}

Therefore, by integrating Lemma~\ref{lemma_1},~\ref{lemma_2}, and~\ref{lemma_3} into~(\ref{conv_bound_1}), we will have:

\begin{flalign}\label{detail_proof}
&\mathbb{E}\left[\mathbb{E}_{Q}\left[\mathbb{E}_{\text{Air}}\left[f(\mathbf{w}^{r+1})-f(\mathbf{w}^{r}) \right]\right]\right]\nonumber \\
\leq & \frac{\eta\theta }{2K}\sum_{k=1}^{K}\sum_{h=0}^{H}\left[-\|\nabla f^r\|^2_2-\|\nabla f^{r,h}_k\|^2_2+L^2\eta^2 H\left[\sigma^2 + H \|\nabla f_k^{r,h}\|^2_2\right]    \right] + \\
& \frac{(q+{p_b})\theta ^2\eta^2L}{2K{p_b}}H\sigma^2+ HL\theta ^2\eta^2\frac{q(2-{p_b}) +K{p_b}}{2K^2{p_b}} \sum_{k=1}^{K} \sum_{h=0}^{H}\|\nabla f^{r,h}_{k}\|^2+ \frac{\theta ^2\eta^2\sigma_z^2L}{2K^2{p_b}^2}\nonumber\\
&=-\frac{\eta\theta  H}{2}\|\nabla f^r\|^2_2 -  \frac{\eta\theta }{2K}(1-L^2\eta^2H^2-  HL\theta \eta\frac{q(2-{p_b}) +K{p_b}}{K{p_b}}) \sum_{k=1}^{K}\sum_{h=0}^{H}\|\nabla f^{r,h}_k\|^2_2\nonumber\\
&+ \frac{\theta \eta^2LH}{2K}(\eta LHK+\frac{({p_b}+q)\theta }{{p_b}})\sigma^2+ \frac{\theta ^2\eta^2\sigma_z^2L}{2K^2{p_b}^2}\nonumber
\end{flalign}

If we set $1-L^2\eta^2H^2- HL\theta \eta\frac{q(2-{p_b}) +K{p_b}}{K{p_b}} \geq 0$, we can get
\begin{flalign}
    \mathbb{E}&\left[\mathbb{E}_{Q}\left[\mathbb{E}_{\text{Air}}\left[f(\mathbf{w}^{r+1})-f(\mathbf{w}^{r}) \right]\right]\right] \leq -\frac{\eta\theta  H}{2}\|\nabla f^r\|^2_2 \nonumber\\
    +& \frac{\theta \eta^2LH}{2K}(\eta LHK+\frac{({p_b}+q)\theta }{{p_b}})\sigma^2+ \frac{\theta ^2\eta^2\sigma_z^2L}{2K^2{p_b}^2}
\end{flalign}

Next, we sum up the above equation over all $R$ communication rounds and get

\begin{flalign}
    &\frac{1}{R}\sum_{r=0}^{R-1}\|\nabla f^r\|^2_2 \leq \frac{2 (f(\mathbf{w}^{0})-f(\mathbf{w}^{*}))}{\eta\theta  HR}  \\
    &+\frac{\eta L}{K}(\eta LHK+\frac{({p_b}+q)\theta }{{p_b}})\sigma^2+ \frac{\theta \eta L}{HK^2{p_b}^2}\sigma_z^2\nonumber\\
    &=\frac{2 (f(\mathbf{w}^{0})-f(\mathbf{w}^{*}))}{\eta\theta  HR} +\frac{\eta\theta  L}{K}\frac{({p_b}+q)}{{p_b}}\sigma^2+ \eta^2L^2H\sigma^2 + \frac{\theta \eta L\sigma_z^2}{HK^2{p_b}^2}\nonumber
\end{flalign}

\subsection{Proof of the convexity of $\Theta _1$ and $\Theta _2$}\label{hessian_proof}

The second-order partial derivative of functions $\Theta _1$ and $\Theta _2$ can be calculated as:

\begin{align}
\begin{split}
    \frac{\partial \Theta _1}{\partial {p_b}} &= -\frac{A_0q}{{p_b}^2H}-\frac{B_0 q}{2 {p_b}^{\frac{3}{2}} H^{\frac{1}{2}}({p_b}+q)^{\frac{1}{2}}}\\
    \frac{\partial^2 \Theta _1}{\partial {p_b}^2} &= \frac{2 q\left(A_0 \sqrt{H}\left(q \sqrt{{p_b}}+{p_b}^{3 / 2}\right) \sqrt{{p_b}+q}+\frac{1}{2} H {p_b} B_0\left({p_b}+\frac{3}{4} q\right)\right)}{H^{3 / 2} {p_b}^{7 / 2}({p_b}+q)^{3 / 2}}\\
    \frac{\partial^2 \Theta _1}{\partial {p_b} \partial H} &= \frac{q\left(A_0 \sqrt{{p_b}} \sqrt{H} \sqrt{{p_b}+q}+\frac{1}{4} B_0 {p_b} H\right)}{H^{5 / 2} \sqrt{{p_b}+q} {p_b}^{5 / 2}}\\
    \frac{\partial \Theta _1}{\partial H} &= -\frac{A_0({p_b}+q)}{{p_b} H^{2}}-\frac{1}{2} \frac{B_0 \sqrt{{p_b}+q} {p_b}}{({p_b} H)^{3 / 2}}\\
    \frac{\partial^2 \Theta _1}{\partial H^2} &= \frac{2\left(A_0({p_b}+q) \sqrt{{p_b} H}+\frac{3}{8} B_0 \sqrt{{p_b}+q} {p_b} H\right)}{\sqrt{{p_b} H} {p_b} H^{3}}\\
    \frac{\partial^2 \Theta _1}{\partial H \partial {p_b}} &= \frac{1}{4} \frac{q(4 A_0 \sqrt{{p_b}+q \sqrt{{p_b}} H+B_0 {p_b} H)}}{{p_b}^{2} H^{2} \sqrt{{p_b}+q} \sqrt{{p_b} H}}\\
    \frac{\partial^2 \Theta _2}{\partial {p_b}^2} &=\lambda \frac{\ln^2{p_b} - 3 \ln {p_b}+1}{{p_b}\ln^2 {p_b}(\ln {p_b}-1)^2}T^{comm}\\
    \frac{\partial^2 \Theta _2}{\partial H^2} &= \frac{\partial^2 \Theta _2}{\partial H\partial {p_b}} = \frac{\partial^2 \Theta _2}{\partial {p_b} \partial H} = 0
\end{split}
\end{align}

We further have $\frac{\partial^2 \Theta _1}{\partial {p_b}^2} \geq 0$ and $\frac{\partial^2 \Theta _1}{\partial {p_b}^2} \times \frac{\partial^2 \Theta _1}{\partial H^2} - \frac{\partial^2 \Theta _1}{\partial {p_b} \partial H} \times \frac{\partial^2 \Theta _1}{\partial H \partial {p_b}} \geq 0$. Thus, both function $\Theta _1(\boldsymbol{\phi})$ and $\Theta _2(\boldsymbol{\phi})$ are positive and convex.

The Hessian matrix of the function $\Theta _1$ can be described as $\begin{bmatrix}
\frac{\partial^2 \Theta _1}{\partial {p_b}^2} &\frac{\partial^2 \Theta _1}{\partial {p_b} \partial H} \\ 
 \frac{\partial^2 \Theta _1}{\partial H \partial {p_b}}&\frac{\partial^2 \Theta _1}{\partial H^2} 
\end{bmatrix}$, where we can find that both $\frac{\partial^2 \Theta _1}{\partial {p_b}^2}$ and $\begin{vmatrix}
\frac{\partial^2 \Theta _1}{\partial {p_b}^2} &\frac{\partial^2 \Theta _1}{\partial {p_b} \partial H} \\ 
 \frac{\partial^2 \Theta _1}{\partial H \partial {p_b}}&\frac{\partial^2 \Theta _1}{\partial H^2} 
\end{vmatrix}$ are positive. In addition, the Hessian matrix of the function $\Theta _1$ is positive-definite, and we can also easily obtain that the Hessian matrix of the function $\Theta _2$ is positive semi-definite. Therefore, both function $\Theta _1(\boldsymbol{\phi})$ and $\Theta _2(\boldsymbol{\phi})$ are positive and convex.


\vfill

\end{document}